\documentclass[10pt,twocolumn,twoside]{IEEEtran}
\ifCLASSINFOpdf
   \usepackage[pdftex]{graphicx}
\else
\fi
\usepackage{epstopdf}

\usepackage{tikz}
\usetikzlibrary{arrows} 
\tikzset{
    >=stealth',
    image/.style={
           rectangle,
		   fill=yellow!10,
           draw=black, very thick,
           text width=10em,
           minimum height=2em,
           text centered},
    process/.style={
           rectangle,
           rounded corners,
		   fill=blue!10,
           draw=black, very thick,
           text width=8em,
           minimum height=2em,
           text centered},
    object/.style={
           circle,
		   fill=yellow!10,
           draw=black, very thick,
           text width=2em,
           minimum height=1.5em,
           text centered},
    pil/.style={
           ->,
           thick,
           shorten <=2pt,
           shorten >=2pt,}
}

\usepackage{color}
\usepackage{colortbl}

\usepackage[cmex10]{amsmath}
\usepackage{amsfonts}
\usepackage{amsthm}
\usepackage{dsfont}
\usepackage{siunitx}
\usepackage{epstopdf}
\usepackage{multirow}
\usepackage{soul}
\usepackage{hyperref}
\usepackage{bm}

\DeclareMathOperator*{\argmin}{arg\,min\,}

\newcommand{\Hm}{\mathbf{H}} 
\newcommand{\y}{\mathbf{y}} 
\newcommand{\x}{\mathbf{x}} 
\newcommand{\n}{\mathbf{n}} 
\newcommand{\z}{\mathbf{z}} 
\newcommand{\vs}{\mathbf{v}} 
\newcommand{\dv}{\mathbf{d}} 
\newcommand{\D}{\mathbf{D}} 
\newcommand{\wt}{\mathbf{w}} 
\newcommand{\A}{\mathbf{A}} 
\newcommand{\B}{\mathbf{B}} 
\newcommand{\C}{\mathbf{C}} 
\newcommand{\e}{\mathbf{e}} 
\newcommand{\f}{\mathbf{f}} 
\newcommand{\M}{\mathbf{M}} 
\newcommand{\I}{\mathbf{I}} 
\newcommand{\G}{\mathbf{G}} 
\newcommand{\K}{\mathbf{K}} 
\newcommand{\T}{\mathbf{T}} 
\newcommand{\uu}{\mathbf{u}} 
\newcommand{\pp}{\mathbf{p}} 
\newcommand{\PP}{\mathbf{P}} 
\definecolor{light}{gray}{.9}
\newtheorem{theorem}{Theorem}
\newtheorem{lemma}{Lemma}
\newtheorem{corollary}{Corollary}

\usepackage{algorithmic}

\usepackage{array}

\ifCLASSOPTIONcompsoc
  \usepackage[caption=false,font=normalsize,labelfont=sf,textfont=sf]{subfig}
\else
  \usepackage[caption=false,font=footnotesize]{subfig}
\fi
\usepackage{url}

\begin{document}
%
\title{A Framework for Fast Image Deconvolution with Incomplete Observations}
%
%
%

\author{Miguel Sim\~{o}es, Luis B.\ Almeida, Jos\'{e} Bioucas-Dias,~\IEEEmembership{Member,~IEEE,} Jocelyn Chanussot,~\IEEEmembership{Fellow,~IEEE,}
\thanks{Miguel Sim\~{o}es is with Instituto de Telecomunica\c{c}\~{o}es, Instituto
Superior T\'{e}cnico, Universidade de Lisboa, Portugal, and GIPSA-Lab, Universit\'{e} de Grenoble, France.}
\thanks{Luis B. Almeida and Jos\'{e} Bioucas-Dias are with Instituto de Telecomunica\c{c}\~{o}es, Instituto
Superior T\'{e}cnico, Universidade de Lisboa, Portugal.}
\thanks{Jocelyn Chanussot is with GIPSA-Lab, Grenoble Institute of Technology, France and with the Faculty of Electrical and Computer Engineering, University of Iceland.}
\thanks{This work was partially supported by the Funda\c{c}\~{a}o para a Ci\^{e}ncia e Tecnologia, Portuguese Ministry of Science and Higher Education, project PEst-OE/EEI/0008/2013 and grant SFRH/BD/87693/2012.}}

%
%

\markboth{Journal of \LaTeX\ Class Files,~Vol.~11, No.~4, December~2012}%
{Shell \MakeLowercase{\textit{et al.}}: Bare Demo of IEEEtran.cls for Journals}
%



\maketitle

\begin{abstract}
In image deconvolution problems, the diagonalization of the underlying operators by means of the FFT usually yields very large speedups. When there are incomplete observations (e.g., in the case of unknown boundaries), standard deconvolution techniques normally involve non-diagonalizable operators, resulting in rather slow methods, or, otherwise, use inexact convolution models, resulting in the occurrence of artifacts in the enhanced images. In this paper, we propose a new deconvolution framework for images with incomplete observations that allows us to work with diagonalized convolution operators, and therefore is very fast. We iteratively alternate the estimation of the unknown pixels and of the deconvolved image, using, e.g., an FFT-based deconvolution method. This framework is an efficient, high-quality alternative to existing methods of dealing with the image boundaries, such as edge tapering. It can be used with any fast deconvolution method. We give an example in which a state-of-the-art method that assumes periodic boundary conditions is extended, through the use of this framework, to unknown boundary conditions. Furthermore, we propose a specific implementation of this framework, based on the alternating direction method of multipliers (ADMM). We provide a proof of convergence for the resulting algorithm, which can be seen as a ``partial'' ADMM, in which not all variables are dualized. We report experimental comparisons with other primal-dual methods, where the proposed one performed at the level of the state of the art. Four different kinds of applications were tested in the experiments: deconvolution, deconvolution with inpainting, superresolution, and demosaicing, all with unknown boundaries. 
\end{abstract}

\begin{IEEEkeywords}
Deconvolution, incomplete observations, convex non-smooth optimization, alternating direction method of multipliers (ADMM), primal-dual optimization, inpainting, superresolution, demosaicing.
\end{IEEEkeywords}

%
\IEEEpeerreviewmaketitle

\section{Introduction}
\label{sec:introduction}
\IEEEPARstart{D}{eblurring} is one of the classical problems of the image processing field. It consists in the recovery of sharp images from blurred ones, where the blur can be any kind of degradation that results in a decrease of image sharpness. The blur can be caused, for example, by camera motion, or by the propagation of light through the atmosphere. For a review of deblurring methods, see~\cite{campisi2007} and the more recent~\cite{Rajagopalan2014}. In this work, we consider the situation in which some pixels of the blurred image are not observed. Examples are the pixels corresponding to the boundaries of the image, saturated or missing pixels, and the extra pixels that would have been observed if the sensors that acquired the images had a higher spatial resolution than the actual sensors that were used, as in the case of superresolution~\cite{Milanfar2010} and demosaicing~\cite{Lukac2008}. We assume that the blur corresponds to a convolution with a known \emph{point spread function} (PSF). Since the blurring PSF is assumed to be known, this corresponds to the so-called non-blind deconvolution problem, but is also of use in blind deconvolution algorithms, since most of these involve non-blind deconvolution as one of the processing steps.

In this paper, bold lowercase letters denote column vectors, and bold uppercase letters denote matrices. $\I_p$ denotes the identity matrix of size $p \times p$. $\mathbf 1_p$ denotes a vector of ones of size $p$, and $\mathbf 0$ denotes a zero vector or matrix of appropriate size. A superscript on a vector, as in $\x^i$, denotes the index of a sequence of vectors. We will use the notation $\{\x^i\}$ as a shorthand for representing the sequence $\{\x^i\}_{i=1}^{+\infty}$. The subgradient operator will be denoted by $\partial$. The inverse of matrix $\A$ will be denoted by $\A^{-1}$ and the $\A$-norm will be denoted by $\| \cdot \|_{\A}$, i.e., $\|\x\|_{\A} = \sqrt{\x^T \A \x}$, where $\A$ is positive-definite and $\x$ is a vector.

\subsection{Problem statement}
\label{sec:problemstatement}

Consider the estimation of a sharp image from a blurred one. Assume that the support of the convolution kernel has size $(2b+1) \times (2b+1)$ pixels and is centered at the origin.\footnote{For simplicity, we consider only kernels supported in square regions centered at the origin. The extension to other situations would be straightforward.} Let the size of the blurred image be $m \times n$. To express that image as a function of the sharp one, we need to consider a region of the sharp image of size $m' \times n'$, with $m'=m+2b$ and $n'=n+2b$; the central $m \times n$ zone of this sharp image is in the same spatial location as the blurred image. Fig.~\ref{fig:scheme} illustrates this situation.

\begin{figure}[tbh!]
	\begin{center}
	\includegraphics[trim={0 10pt 0 10pt},clip,scale=.18]{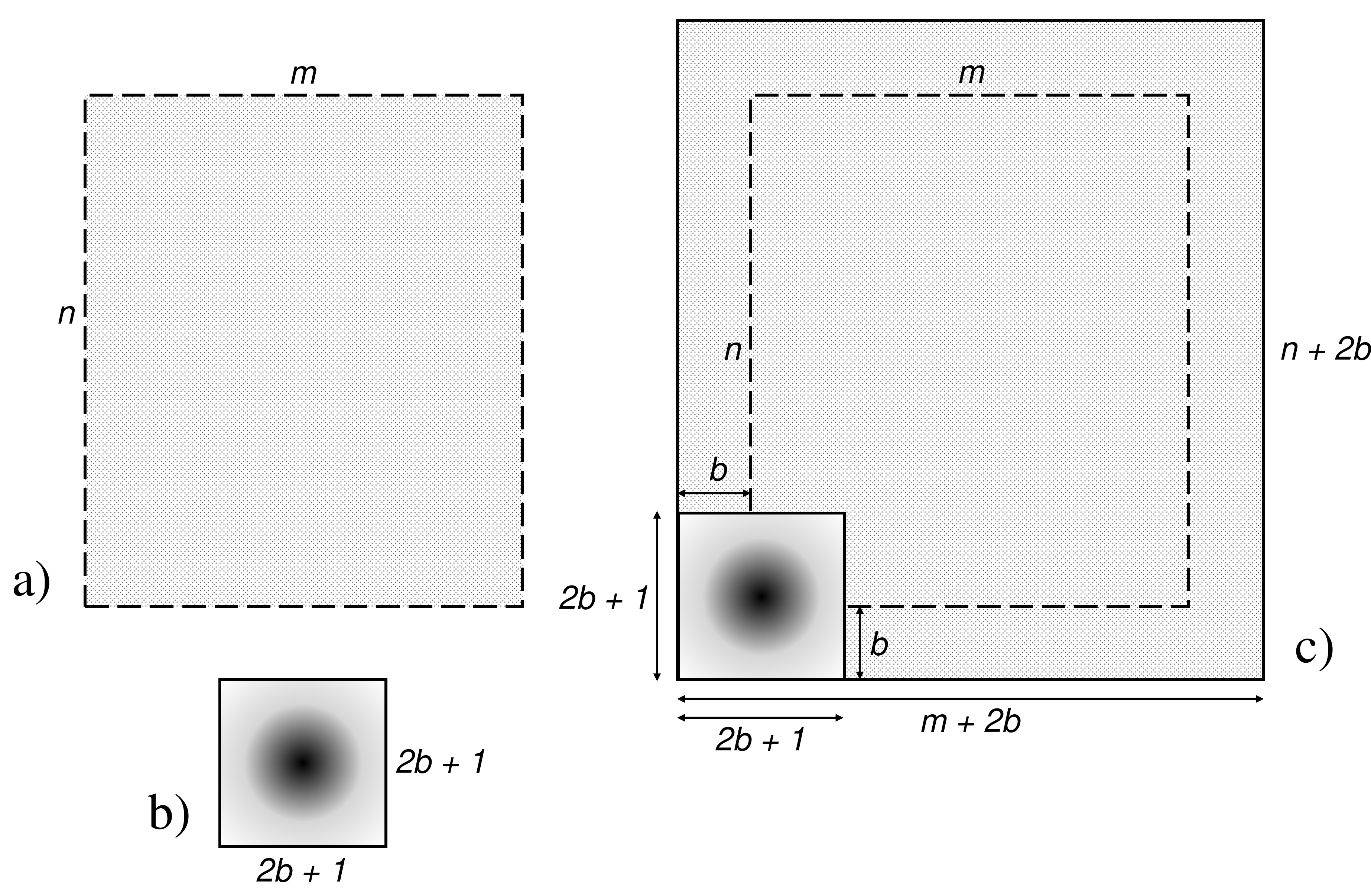}%
	\vspace{-10pt}
	\caption{An illustration of the dimensions of the images involved in typical deblurring problems. a) Blurred image. b) Blurring kernel. c) Sharp image.}
	\label{fig:scheme}
	\vspace{-15pt}
	\end{center}	
\end{figure}

Let $\mathcal{X}$ be a given nonempty convex subset of $\mathbb{R}^{m'n'}$. Assuming a linear observation model with additive noise, we can express the blurring operation as
\begin{equation} \label{model1}
	\y = \T \x + \n,
\end{equation}
in which the images are represented by column vectors with the pixels arranged in lexicographic order, $\y \in \mathbb R^{mn}$ is the observed image,  $\x \in \mathcal{X}$ is the sharp image, $\T$ is a $mn \times m'n'$ block-Toeplitz-Toeplitz-block (BTTB) matrix such that $\T \x$ represents the convolution of the sharp image with the blurring PSF, and $\n \in \mathbb R^{mn}$ is the observation noise. 

In this paper, we address the estimation of the sharp image $\x$ from the blurred one, in situations in which some of the pixels of the blurred image are unobserved. This situation occurs, for example, in inpainting and superresolution problems. Demosaicing is a form of superresolution, and also involves unobserved pixels. Often, in simple deconvolution problems, one is only interested in estimating the central $m \times n$ region of $\x$, which we shall designate by \emph{cropped sharp image}, denoted by $\bar \x$. If one is interested in estimating the whole image $\x$, this can be seen as an inpainting problem, since we are estimating the pixels of $\x$ in a boundary zone of width $b$ around the central $m \times n$ region, and this zone is not present in $\y$. Even if we do not wish to estimate this zone, the need to properly handle the boundary zone still exists in most real-life deconvolution methods, as we discuss next.

A difficulty with the use of model \eqref{model1} is that deconvolution methods based on it normally involve products by large BTTB matrices and/or the inversion of such matrices, and both operations are computationally heavy. In order to obtain fast deconvolution methods, many authors replace \eqref{model1} with models that involve simpler computations. One of the most frequently used models is
\begin{equation} \label{model2}
	\y = \overline \T \bar \x' + \n,
\end{equation}
in which $\overline \T$ is a block-circulant-circulant-block (BCCB) matrix of size $mn \times mn$, and $\bar \x' \in \mathbb R^{mn}$ is an approximation of the true cropped sharp image $\bar \x$. $\overline \T \bar \x'$ represents the circular convolution of $\bar \x'$ with the blurring PSF. The speed advantage of \eqref{model2} comes from the fact that BCCB matrices are diagonalized by the two-dimensional discrete Fourier transform (DFT), and therefore products and inverses involving such matrices can be efficiently computed using the FFT. However, the fact that $\bar \x'$ is only an approximation of the true image $\bar \x$ means that \eqref{model2} is not an exact model of the convolution process. As a consequence of this, the sharp images obtained by these methods normally exhibit artifacts, typically in the form of ringing. The use of model \eqref{model2} is often referred to as the use of \emph{periodic boundary conditions}, because it is equivalent to the use of \eqref{model1} with the true sharp image $\x$ replaced with an image obtained by periodically repeating $\bar \x$ in the horizontal and vertical directions, and then retaining only the central $m' \times n'$ region of the resulting periodic image. Other possibilities exist for obtaining fast deconvolution methods, besides the use of periodic boundary conditions. For example, one can use \emph{reflexive} or \emph{anti-reflexive} boundary conditions, which, under appropriate assumptions, lead to matrices that are diagonalizable, respectively, by the 2D discrete cosine transform and the 2D discrete sine transform, and therefore also yield significant speed advantages~\cite{Ng1999, Hansen2006, Donatelli2010}. However, the use of any of these boundary conditions (and, in fact, the use of any artificially imposed boundary conditions) corresponds to the use of an inexact convolution model, and therefore gives rise to artifacts.

The occurrence of artifacts can be completely eliminated by the use of an exact model of the convolution process. 
A relatively recent method, which we shall designate by AM \cite{Almeida2013a,Matakos2013}, uses a model of the form
\begin{equation} \label{model3}
	\y = \M \widetilde \T \x + \n,
\end{equation}
in which $\widetilde \T$ is an $m'n' \times m'n'$ BCCB matrix that corresponds to a circular convolution with the blurring PSF, and $\M$ is an $mn \times m'n'$ masking matrix that selects, from $\widetilde \T \x$, only the pixels that correspond to the observed image, discarding a boundary zone of width $b$ in the periphery of the image $\widetilde \T \x$. The circular convolution $\widetilde \T \x$ in \eqref{model3} only differs from the linear (i.e., non-circular) convolution $\T \x$ of \eqref{model1} by the presence of that boundary zone, and therefore Eq.~\eqref{model2} is an exact model of the convolution process. Computationally, this method has the advantage of using a diagonalizable matrix, $\widetilde \T$, but needs to deal with the fact that $\M \widetilde \T$ is not easily diagonalizable. This difficulty is circumvented, in AM, through an adaptation of the Alternating Directions Method of Multipliers (ADMM)~\cite{Afonso2011, Boyd2011}. By means of the splitting of a variable, AM decouples the matrix $\M$, which is diagonalizable in the spatial domain, from $\widetilde \T$, which is diagonalizable in the frequency domain, thereby allowing a significant speedup to be achieved.

The deconvolution framework proposed in the present work uses a different, but also exact, convolution model. The convolution process is modeled as
\begin{equation} \label{model4}
	\tilde \y = \widetilde \T \x + \n,
\end{equation}
in which $\widetilde \T$ is the same BCCB matrix as in \eqref{model3}, and $\tilde \y \in \mathbb R^{m'n'}$ represents the observed image $\y$ surrounded by the boundary region of width $b$ mentioned in the previous paragraph. In the proposed framework, this boundary region is estimated, instead of being masked out, as happened in AM.  We will present two implementations of this framework: an implementation using an off-the shelf deconvolution method that assumes circular boundary conditions, and an efficient implementation using a ``partial'' ADMM in which not all variables are dualized, and for which we present a convergence theorem.

Deconvolution methods that do not impose boundary conditions, such as AM and the method proposed in this paper, are often referred to as methods that use \emph{unknown boundaries}. We will use that nomenclature in this paper.


\subsection{Related work}

As mentioned above, many of the published deconvolution methods, such as~\cite{Ng1999, Hansen2006, Donatelli2010, Danielyan2012}, use artificially imposed boundary conditions, leading to the occurrence of artifacts. Some methods try to reduce the intensity of those artifacts by preprocessing the observed image's borders, e.g., through edge tapering or smoothing. A brief review of preprocessing methods is given in~\cite{Reeves2005}. Some other methods, such as~\cite{Jia2008, Fan2011}, consider the blurred image as being extended with a boundary zone as in~\eqref{model4}, but, instead of correctly estimating that zone, use a synthetic extension that is relatively easy to compute. Since the contents of that extension are artificially imposed, these methods give rise to artifacts.

Methods that use exact models have the potential to completely eliminate the occurrence of artifacts. One such method is the one proposed in~\cite{Reeves2005}, which implicitly uses model \eqref{model4}, although that model is not explicitly mentioned in the paper. That method is limited to the use of quadratic regularizers, which allow a fast implementation but yield relatively low-quality deconvolution results. The method has been extended in~\cite{Sorel2012}, allowing the use of more general regularizers. In the latter form, the method achieves a relatively high speed by limiting the estimation of the boundary zone (an estimation that is time-consuming, in that method's formulation) to a few initial iterations of the optimization procedure. This version of the method improves on the results of~\cite{Reeves2005} due to the use of more appropriate regularizers, but the imperfect estimation of the boundary zone gives rise to artifacts in the deblurred images. Similar approaches for astronomical images can be found in~\cite{Bertero2005, Vio2005}. Another method, proposed in~\cite{Chan2004}, uses model \eqref{model1}, and is rather slow due to the use of non-diagonalizable BTTB matrices.

As previously said, the AM method~\cite{Almeida2013a,Matakos2013} uses model \eqref{model3} and is based on ADMM. Compared to other algorithms, the use of ADMM that is made in that method has some convenient properties: its convergence is guaranteed under rather general conditions, and each iteration only involves operations with a low computational cost (e.g., the inversion of diagonalizable matrices). ADMM is a primal-dual method, which means that it solves a primal optimization problem as well as its dual convex formulation, in the sense of the Fenchel-Rockafellar duality theory (see~\cite{Combettes2009, Parikh2013, Komodakis2014} for recent reviews of proximal and primal-dual methods). In Section~\ref{sec:exp}, we experimentally compare AM with the ADMM-based method proposed in the present paper, and find the performances of both to be similar.

The works \cite{Condat2013} and \cite{Combettes2014} propose primal-dual methods that do not involve the inversion of large BTTB matrices. In Section~\ref{sec:exp}, we experimentally test one of these methods~\cite{Condat2013}, and find it to be rather slow. Another primal-dual method that considers non-periodic boundaries was proposed in~\cite{OConnor2014}. It needs the structure of the boundaries to be known \textit{a priori}, which normally is not the case in real-life situations. The use of an artificially chosen structure results, once again, in the occurrence of artifacts.

In \cite{Ferreira2010}, a method that has some resemblance to our proposed deconvolution framework was introduced, in the context of the solution of systems of linear equations with Toeplitz system matrices. 


\subsection{Contributions and outline}

This work has two main contributions:
\begin{itemize}
	\item We propose a new framework for solving deconvolution problems with unobserved pixels. This framework is an efficient, high-quality alternative to the use of heuristic methods, such as edge tapering, to reduce the artifacts produced by deconvolution methods that assume periodic boundary conditions.  We give an example of how this framework can be used to extend a state-of-the art deconvolution method to the use of unknown boundaries.
  
   \item  The proposed framework can also be used to develop new deconvolution methods. We propose a specific ADMM-based implementation of the framework, for which we give a proof of convergence. We experimentally compare it with some state-of-the-art methods.
\end{itemize}

The structure of this paper is as follows. Section~\ref{sec:method} describes the proposed framework and the ADMM-based implementation, and gives the convergence proof for the latter. Section~\ref{sec:exp} presents experimental results, and has two parts: the first one illustrates the use of the proposed framework to convert an off-the-shelf, state-of-the-art deblurring method (IDD-BM3D~\cite{Danielyan2012}) to the use of unknown boundaries; the second part presents a comparison between the proposed ADMM-based method and some other published ones, in problems of deconvolution with and without inpainting, of superresolution, and of demosaicing, all with unknown boundaries. Section~\ref{sec:conclusion} concludes.

\section{The proposed framework} \label{sec:method}
\label{sec:framework}

\subsection{Basic structure}
\label{sec:basicstructure}

As previously mentioned, the framework that we propose is based on model~\eqref{model4}. From here on, we'll express the extended blurred image $\tilde \y$ in a form that is more convenient for the treatment that follows, and that encompasses not only the case of unknown boundaries, but also all the other cases of unobserved pixels. We'll denote the number of observed pixels of the blurred image by $k$, and the number of unobserved pixels (including the above-mentioned boundary zone) by $d$. Let $\z \in \mathcal Z$ and $\y \in \mathbb R^k$ be  column vectors containing, respectively, the elements of $\tilde \y$ that correspond to unobserved pixels and those that correspond to observed pixels; $\mathcal{Z}$ is some given nonempty convex subset of $\mathbb{R}^{d}$. In a simple deconvolution problem with unknown boundaries, $\z$ will contain the boundary zone, and $\y$ will contain the observed blurred image. In a combined deconvolution and inpainting problem, $\z$ will contain both the boundary zone and the additional unobserved pixels, and $\y$ will contain the pixels of the blurred image that were actually observed. We will reorder the elements of the extended image $\tilde \y$ as $\left[\begin{smallmatrix} \y \\ \z \end{smallmatrix}\right] = \PP \tilde \y$, where $\PP$ is an appropriate permutation matrix, so that the observed pixels are in the first positions and the unobserved pixels are in the last positions of the vector $\left[\begin{smallmatrix} \y \\ \z \end{smallmatrix}\right]$.

Conceptually, the proposed framework is rather simple. It consists of using the blurring model~\eqref{model4}, and alternately estimating $\x$ and $\z$, as shown in Fig.~\ref{alg:image_estimation}. In this framework, step 2, which estimates $\x$, can be performed, essentially, with any existing deblurring method that assumes circular boundary conditions.%
\footnote{The framework can also be used with methods that use other boundary conditions. The only difference will be in the structure of the matrix $\widetilde \T$ used in model~\eqref{model4}. Instead of being a BCCB matrix, it will have the proper structure for the boundary conditions under consideration.}
Step 3 is performed by just computing the reconstructed blurred image, given by $\widetilde \T \x$, and selecting from it the pixels that correspond to $\z$.

An important difference of the proposed framework relative to most published deblurring methods, including AM, is that, in this framework, the unobserved pixels of the blurred image (represented by $\z$) are explicitly estimated. This means that we have one more variable to estimate ($\z$).

\looseness -1 As given in Fig.~\ref{alg:image_estimation}, the proposed framework is rather general. It can be used to design new deblurring methods, an example of which is the efficient ADMM-based method that we propose in Section~\ref{sec:partialadmm}. It is also an efficient, high-quality alternative to methods such as edge tapering, to convert existing deblurring methods that impose specific boundary conditions into methods that work with unknown boundaries. We illustrate this in Section~\ref{iddbm3d}, by using the proposed framework to convert an off-the-shelf, state-of-the-art deblurring method that assumes circular boundary conditions (IDD-BM3D) into a method that uses unknown boundaries.

\begin{figure}[tb]
	\begin{center}
		\colorbox{light}{\parbox{1.0\columnwidth}{
				\begin{algorithmic}
					\STATE{\centering \textbf{Framework 1} \nonumber \\}
					\STATE{
						{\small 1.} Initialize $\x^1$ and $\z^1$. Let $i=1$.}
					\REPEAT
					\STATE{
						{\small 2.} Compute $\x^{i+1}$ given $\x^{i}, \, \z^{i}$.}
					\STATE{
						{\small 3.} Compute $\z^{i+1}$ given $\x^{i+1}$.}
					\STATE{
						{\small 4.} Increment $i$.}
					\UNTIL{stopping criterion is satisfied.}
				\end{algorithmic}
			}}
		\caption{The proposed deblurring framework.}
		\label{alg:image_estimation}
	\end{center}
	\vspace{-15pt}	
\end{figure}

\subsection{ADMM-based method}
\label{sec:partialadmm}

\looseness -1 ADMM has been widely used, in recent years, to solve high-dimensional problems in signal and image processing, due to its ability to yield high-quality solutions in a computationally efficient way, in many practical situations. In its original formulation, ADMM can be used to solve problems of the form
\begin{equation*}
\underset{\uu}{\text{minimize}} \quad f(\uu) + \psi(\K \uu),
\end{equation*}
where both $f$ and $\psi$ are closed, proper convex, possibly non-smooth functions (see Appendix~\ref{sec:app_admm} for details). In comparison with some other primal-dual methods, in the context of image deconvolution, ADMM uses an extra set of variables and involves a matrix inversion. These characteristics, however, apparently are the ones that allow it to remain quite competitive relative to more recent primal-dual methods, as illustrated, for example, by the results that we report in Sections~\ref{sec:deblurring} to \ref{sec:appraisal}.

In the ADMM-based implementation of Framework 1 that we propose, we'll use the blurring model of Eq.~\eqref{model4}. The noise $\n$ will be assumed to be i.i.d. Gaussian. We'll use a maximum-a-posteriori (MAP) formulation, and consequently, the data-fitting term of our objective function will be given by
\begin{equation} \label{eq:new_model}
f(\x,\z) = \frac{1}{2} \Bigg\|\begin{bmatrix}
\y \\
\z
\end{bmatrix} - \Hm \x \Bigg\|^2,	
\end{equation}
with $\Hm = \PP \widetilde \T$ and $\x \in \mathcal X$, where $\mathcal X$ is some given convex subset of $\mathbb{R}^{k+d}$. 

The problem to be solved will be expressed as
\begin{equation} \label{eq:problem_xz}
\begin{aligned}
& \underset{\x, \z}{\text{minimize}}
& & f(\x, \z) + \phi(\D \x),
\end{aligned}
\end{equation}
where $\D \in \mathbb R^{l \times (k+d)}$ is a matrix that extracts a linear representation of the estimated image, such as edges, $l$ is the number of components of that representation, and $\phi(\D \x)$ is a regularizer that promotes some desirable characteristic of images, such as sharp edges.

%

We'll start by considering the use of ADMM in its standard form to solve problem~(\ref{eq:problem_xz}). The resulting method will not be very efficient, because it will involve a step that is computationally heavy, but it will be useful to motivate the method that we propose, and to analyze some of its properties. We'll then describe our proposed method, which avoids the above-mentioned computational inefficiency through the use of a \emph{partial} ADMM.%
\footnote{We qualify it as ``partial'' because not all variables are dualized.}

In what follows, we will make use of the variables $\vs, \dv \in \mathbb R^{l+d}$, decomposed as $\vs = \big[ \begin{smallmatrix} \vs_x \\ \vs_z \end{smallmatrix} \big]$ and $\dv = \left[\begin{smallmatrix}\dv_x\\ \dv_z \end{smallmatrix}\right]$, with $\vs_x, \dv_x \in \mathbb R^l$ and $\vs_z, \dv_z \in \mathbb R^d$.  To apply ADMM to problem~(\ref{eq:problem_xz}), we first define
\[
\begin{aligned}
\uu &= \begin{bmatrix}
  \x \\
  \z
\end{bmatrix},
&
\K &= \begin{bmatrix}
  \D & \mathbf{0} \\
  \mathbf{0} & \I_{d}
\end{bmatrix},
&
&\bar f(\uu) &= f(\x,\z),
\end{aligned}
\]
and also define
\begin{equation*}
	\psi(\vs) = \phi(\vs_x),
\end{equation*}
so that $\psi(\K \uu) = \phi(\D \x)$.

We rewrite~\eqref{eq:problem_xz} as
\begin{equation} \label{eq:problem_u}
\begin{aligned}
& \underset{\uu, \vs}{\text{minimize}}	& & \bar f(\uu) + \psi(\vs) \\
&\text{subject to}											& & \vs = \K \uu.
\end{aligned}
\end{equation}
\noindent Applying ADMM to this problem, we obtain an iteration of the sequence of steps
\begin{align}
\left[\begin{matrix} \x^{i+1} \\ \z^{i+1} \end{matrix} \right] &\in \underset{\x, \z}\argmin \;  
f(\x, \z) + \frac{\mu}{2} \Bigg\| \vs^i - \K \begin{bmatrix} \x \\ \z \end{bmatrix} - \dv^i \Bigg\|^2, \label{eq:admm_xz_1} \\
\vs^{i+1} &\in \underset{\vs}\argmin \;
\psi (\vs) + \frac{\mu}{2} \Bigg\| \vs - \K \begin{bmatrix} \x^{i+1} \\ \z^{i+1} \end{bmatrix} - \dv^i \Bigg\|^2, \label{eq:admm_xz_2} \\
\dv^{i+1} &= \dv^i - \Bigg(\vs^{i+1} - \K \begin{bmatrix} \x^{i+1} \\ \z^{i+1} \end{bmatrix}\Bigg) \label{eq:admm_xz_3}.
\end{align}
From now on, we will refer to the iteration of steps~\eqref{eq:admm_xz_1}--\eqref{eq:admm_xz_3} as the \textit{standard ADMM}.

As mentioned above, the standard ADMM won't normally be computationally efficient. This is due to the fact that, in step~(\ref{eq:admm_xz_1}), $\x$ and $\z$ need to be estimated simultaneously, and this will normally involve the inversion of a large matrix that is not easily diagonalizable. For the (rather small) images used in the experimental tests of \mbox{Section~\ref{sec:exp}}, which have $256 \times 256$ pixels, this would involve a matrix with $(256 \times 256)^2 \approx 4 \times 10^9$ elements. Since matrix inversion runs in $O[n^3]$ time, directly inverting matrices as large as these would not be feasible, in useful time, with current computers. Furthermore, it would be impracticable to manipulate such large matrices, given the memory sizes of most present-day computers. For the much larger images that are commonly used in practice, the difficulties would be even larger.

To motivate the solution that we propose, we note that if, in problem~\eqref{eq:problem_xz}, we consider minimizing relative to $\x$ and to $\z$ separately, only the minimization relative to $\x$ will be difficult to perform. The minimization relative to $\z$ will be easy to implement in a computationally efficient way, because it is the minimization of a quadratic function, and the matrix $\Hm = \PP \widetilde \T$ is diagonalizable in the frequency domain (with an appropriate permutation, corresponding to the product by $\PP$). In view of this, we will separate the minimization relative to $\x$ from the minimization relative to $\z$, applying them in an alternating manner, and we will apply the ADMM machinery only to the variable $\x$, instead of applying it to $\big[ \begin{smallmatrix} \x \\ \z \end{smallmatrix} \big]$, as happened in the standard ADMM. Of course, the convergence guarantees of the standard ADMM won't apply to the proposed method. We will, therefore, present a convergence proof for it.

Since we are applying the ADMM machinery only to $\x$, step~\eqref{eq:admm_xz_1} of the standard ADMM will be replaced by a minimization of
\begin{equation} \label{eq:admm_xz_1a}  
	f( \x, \z) + \frac{\mu}{2} \Big\| \vs_x^i - \D \x -  \dv_x^i \Big\|^2,
\end{equation}
which we will solve approximately by means of an alternating minimization on $\x$ and $\z$ through one or more block-Gauss-Seidel (BGS) passes. Furthermore, steps~\eqref{eq:admm_xz_2} and \eqref{eq:admm_xz_3} will have to be modified so as to refer only to $\x$, and not to $\left[\begin{smallmatrix}\x\\ \z \end{smallmatrix}\right]$. If we use just one BGS pass to minimize~\eqref{eq:admm_xz_1a}, the complete method will correspond to the iteration of
\begin{align}
\x^{i+1} &\in \underset{\x}\argmin \; f( \x, \z^i) + \frac{\mu}{2} \Big\| \vs_x^i - \D \x -  \dv_x^i \Big\|^2, \label{eq:admm_aprox_1} \\
\z^{i+1} &\in \underset{\z}\argmin \; f( \x^{i+1}, \z), \label{eq:admm_aprox_2} \\
\vs_x^{i+1} &\in \underset{\vs_x}\argmin \; \phi(\vs_x) + \frac{\mu}{2} \Big\| \vs_x - \D \x^{i+1} -  \dv_x^i \Big\|^2,  \label{eq:admm_aprox_3} \\
\dv_x^{i+1} &= \dv_x^i - (\vs_x^{i+1} - \D\x^{i+1}). \label{eq:admm_aprox_4}
\end{align}
If we use more BGS passes, instead of just one, there will be an inner loop consisting of steps \eqref{eq:admm_aprox_1} and \eqref{eq:admm_aprox_2}. 

As can easily be seen, this method falls within the scope of Framework 1, the main steps being \eqref{eq:admm_aprox_1} and \eqref{eq:admm_aprox_2}; steps \eqref{eq:admm_aprox_3} and \eqref{eq:admm_aprox_4} are added by the use of the ADMM technique. We will call the iteration of \eqref{eq:admm_aprox_1}--\eqref{eq:admm_aprox_4} (with one or more BGS passes) the \textit{partial ADMM}; this designation stems from the fact that we only apply the ADMM technique to $\x$, and not to $\left[\begin{smallmatrix}\x\\ \z \end{smallmatrix}\right]$.

We will now address the issue of the convergence of the partial ADMM. We will start by proving (in Theorem~\ref{th:main}) the convergence of a somewhat more general method, and we will then show (in Corollary~\ref{th:corollary}) that the partial ADMM is a special case of that method, and is therefore encompassed by Theorem~\ref{th:main}.

Until now, we have assumed the data-fitting term $f$ to be given by~\eqref{eq:new_model}. For the proof of convergence, we will allow $f$ to have the more general form
\begin{equation} \label{eq:quadratic_f}
f(\x, \z) = \frac{1}{2} \begin{bmatrix} \x \\ \z \end{bmatrix}^T \begin{bmatrix} \A & \B \\ \B^T & \C \end{bmatrix} \begin{bmatrix} \x \\ \z \end{bmatrix} + \begin{bmatrix} \x \\ \z \end{bmatrix}^T \begin{bmatrix} \e \\ \f \end{bmatrix} + g,
\end{equation}
where $\A \in \mathbb{R}^{(k+d) \times (k+d)}$, $\B \in \mathbb{R}^{(k+d) \times d}$, $\C \in \mathbb{R}^{d \times d}$, $\e \in \mathbb{R}^{k+d}$, $\f \in \mathbb{R}^{d}$, and $g \in \mathbb{R}$, and where we assume that $\C$ is positive-definite (PD) and that $\A - \B \C^{-1} \B^T$ is positive-semidefinite (PSD). These assumptions guarantee that $f$ is convex, and are not very restrictive. The set of functions that they encompass is only slightly less general than the set of all convex quadratic functions. To obtain the latter set, the assumption on $\C$ would have to be relaxed to being PSD, but additional assumptions would need to be made (see, e.g., Appendix A.5.5 of~\cite{Boyd2004}).

The convergence result is given by the following theorem:

\begin{theorem} \label{th:main}
	Assume that, in problem~\eqref{eq:problem_xz} with $f$ defined by~\eqref{eq:quadratic_f}, $\C$ is PD, $\A - \B \C^{-1} \B^T$ is PSD, $\D$ is full column rank, and $\phi$ is closed proper convex and coercive. Define $\K = \left[\begin{smallmatrix} \D & \mathbf{0} \\ \mathbf{0} & \I_{d} \end{smallmatrix} \right]$. Then, the set of solutions of problem~\eqref{eq:problem_xz} is non-empty, the sequence $\big\{\big[\begin{smallmatrix} \x^i \\ \z^i \end{smallmatrix}\big]\big\}$ generated by the partial ADMM converges to a solution of that problem, and the sequence $\{\vs^i\}$ converges to $\K \left[\begin{smallmatrix} \x^* \\ \z^* \end{smallmatrix}\right]$, where $\left[\begin{smallmatrix} \x^* \\ \z^* \end{smallmatrix}\right]$ is the limit of $\big\{\big[\begin{smallmatrix} \x^i \\ \z^i \end{smallmatrix}\big]\big\}$.
\end{theorem}

\begin{proof}
The proof is given in Appendix \ref{sec:proof}.
\end{proof}

\begin{corollary} \label{th:corollary}
	For problem~\eqref{eq:problem_xz} with $f$ given by~\eqref{eq:new_model}, Theorem~\ref{th:main} applies. 
\end{corollary}

\begin{proof}
	If we make $\A = \Hm^T \Hm$, $\B^T = - \M_z \Hm$, $\C = \I_d$, $\e^T = - \y^T \M_y \Hm$, $\f = \mathbf{0}$, and $g = \frac{1}{2} \y^T \y$, where $\M_y = [\I_k~\mathbf 0]$ is of size $k \times (k+d)$ and $\M_z = [\mathbf 0~\I_d]$ is of size $d \times (k+d)$, the conditions of Theorem~\ref{th:main} are satisfied.
\end{proof}
	
Regarding the practical implementation of the partial ADMM, the solutions of the minimization problems that constitute steps~\eqref{eq:admm_aprox_1} and \eqref{eq:admm_aprox_2}, with $f$ given by~\eqref{eq:new_model}, are given, respectively, by
\begin{equation} \label{eq:dual_x_iter_resolved_boundaries}
\begin{aligned}
\x^{i+1} &= \big[\Hm^T \Hm + \mu \D^T \D \big]^{-1}\\
& \qquad \big[\Hm^T \M_z^T \z^i + \Hm^T \M_y^T \y + \mu \D^T (\vs_x^i - \dv_x^i) \big]\\
&= \big[\Hm^T \Hm + \mu \D^T \D \big]^{-1}\\
& \qquad \Bigg[\Hm^T \begin{bmatrix}
		\y \\
		\z^i
	\end{bmatrix} + \mu \D^T (\vs_x^i - \dv_x^i) \Bigg]
\end{aligned}	
\end{equation}
and
\begin{equation} \label{eq:dual_z_iter_resolved_boundaries}
\z^{i+1} = \M_z \Hm \x^{i+1}.
\end{equation}
Matrices $\M_y$ and $\M_z$ are defined in the proof of \mbox{Corollary~\ref{th:corollary}}. They are masking matrices and, in particular, $\M_y$ is equivalent to the masking matrix $\M$ used in the AM method~\mbox{\eqref{model3}}, i.e., $\M = \M_y \PP$.

In practice, we have found it useful to use over-relaxation with a coefficient of 2 in the update of $\z$, and therefore we'll replace \mbox{\eqref{eq:dual_z_iter_resolved_boundaries}} with
\begin{equation}
\z^{i+1} = 2 \M_z \Hm \x^{i+1} - \z^i.
\end{equation}

In (\ref{eq:dual_x_iter_resolved_boundaries}), $\Hm^T\Hm = \widetilde \T^T \PP^T \PP \widetilde \T = \widetilde \T^T \widetilde \T$ is a BCCB matrix. If $\D^T \D$ is also BCCB, the matrix inverse in~(\ref{eq:dual_x_iter_resolved_boundaries}) can be efficiently computed by means of the FFT. On the other hand, since $\Hm = \PP \widetilde \T$, products by $\Hm$ or $\Hm^T$ can be computed as products by the BCCB matrix $\widetilde \T$ followed or preceded by the appropriate permutation, and therefore can also be efficiently computed by means of the FFT. Consequently, the iterations of the proposed method are computationally efficient, having complexity $O[(k+d) \log (k+d)]$.

The MATLAB code for the proposed ADMM-based method is available at https://github.com/alfaiate/DeconvolutionIncompleteObs.

\section{Experimental study} \label{sec:exp}

This section has two parts. The first one, corresponding to Subsection \ref{iddbm3d}, illustrates the use of the framework of Fig.~\ref{alg:image_estimation} with a state-of-the-art deblurring method that assumes periodic boundary conditions, to adapt it to the unknown boundaries situation. We plug into the mentioned framework, without modification, a fast, high-quality FFT-based deconvolution method: IDD-BM3D~\cite{Danielyan2012}. As a result of the incorporation into the framework, we obtain a deconvolution method that uses unknown boundaries, and that still retains the speed of FFT-based matrix operations. We compare the results of the resulting method with those of the commonly adopted solution of using edge tapering to deal with the unknown boundaries. The second part of this section, consisting of Subsections \ref{sec:deblurring} to \ref{sec:appraisal}, presents experimental results on the use of the proposed ADMM-based method, and comparisons with AM and with another state-of-the-art method. 

In the deblurring tests described ahead, we created each blurred image by performing a circular convolution of the corresponding sharp image with the desired blurring kernel, and then cropping the result, keeping only the part of the circular convolution that coincided with the linear convolution. We then normalized the images so that all pixel values were in the interval $[0,1]$.  Finally, we added i.i.d. Gaussian noise with a blurred signal-to-noise ratio (BSNR) of 50 dB, unless otherwise indicated.

\subsection{Use of Framework 1 with IDD-BM3D}
\label{iddbm3d}

As previously mentioned, step 2 of Framework 1 can be implemented using, essentially, any deconvolution method. One such method is the state-of-the-art IDD-BM3D~\cite{Danielyan2012}, which assumes circular boundary conditions to be able to perform fast matrix operations in the frequency domain by means of the FFT. It takes a frame-based approach to the deconvolution problem, and performs both a deconvolution and a denoising step. Additionally, it runs another deconvolution method~\cite{Dabov2007} for initialization. We used the published IDD-BM3D software, without change, to implement step 2 of Framework 1. As mentioned in Section~\ref{sec:basicstructure}, step 3 was performed by computing $\widetilde \T \x$ and selecting from it the pixels that corresponded to $\z$.

In the experimental tests, we compared four different situations: (a) direct application of IDD-BM3D without the use of Framework 1, to assess the effect of the method's assumption of circular boundary conditions on an image that didn't obey those conditions; (b) similar to (a), but preprocessing the observed image by edge tapering, to reduce the effect of the mentioned assumption; (c) the use of IDD-BM3D within Framework 1, as described above; and (d) the use of the \emph{partial ADMM} method (the latter situation was included for completeness, but is not essential to the comparison being made here). In situations (a) and (b), IDD-BM3D was run for 400 iterations; in situation (c), we performed 80 iterations of the main loop of Framework 1, and within each of these, IDD-BM3D was run for 5 iterations, giving, again, a total of 400 iterations of IDD-BM3D; in situation (d), we used the Proposed-AD version of the partial ADMM method (see the next section for details).

We ran the experiment using the \emph{cameraman} image with size $256 \times 256$ pixels, blurred with a $9 \times 9$ boxcar filter, and with additive i.d.d. Gaussian noise with a BSNR of 40 dB. In situation (b), the image borders were smoothed with MATLAB's \emph{edgetaper} function with an $11 \times 11$ boxcar filter. This size of the filter was experimentally chosen so as to yield the best final results in terms of improvement in signal-to-noise ratio (ISNR). In situations (a) and (b), IDD-BM3D estimated sharp images of size $m \times n$. In situations (c) and (d), the use of Framework 1 led to estimated images of size $m' \times n'$. To ensure a fair comparison, only the central $m \times n$ regions of the estimated images were used to compute the ISNR, in cases (c) and (d).

The results of the tests, along with the corresponding ISNR values, are shown in Fig.~\ref{fig:cameraman_bm3d}. As can be seen, using IDD-BM3D without taking into account that the boundary conditions were not circular (situation (a)) produced very strong artifacts. With edge tapering (situation (b)), the artifacts were much reduced, although some remained visible; there also was some loss of detail near the image borders. With the use of IDD-BM3D within Framework 1 (situation (c)), there were barely any artifacts, and the image remained sharp all the way to the borders. The values of the ISNR agree with these observations. The partial ADMM method (situation (d)) yielded somewhat lower ISNR than IDD-BM3D within Framework 1. This agrees with the fact that IDD-BM3D is among the best existing image deblurring methods, with a quality that is hard to equal.

\begin{figure*}[!t]
	\centering
	\captionsetup[subfloat]{justification=centering}
	\subfloat[]{\includegraphics[scale=.5]{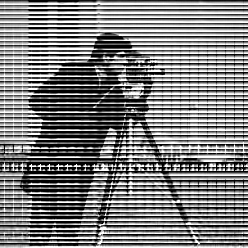}%
		\label{fig:cameraman_bm3d_valid}}
	\hfil
	\captionsetup[subfloat]{justification=justified}
	\subfloat[]{\includegraphics[scale=.5]{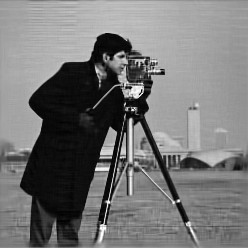}%
		\label{fig:cameraman_bm3d_valid_edgetaper}}
	\hfil
	\subfloat[]{\includegraphics[scale=.5]{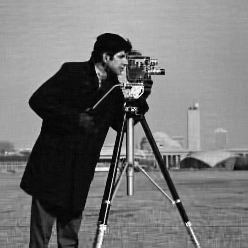}%
		\label{fig:cameraman_bm3d_valid_boundaries}}
	\hfil
	\subfloat[]{\includegraphics[scale=.5]{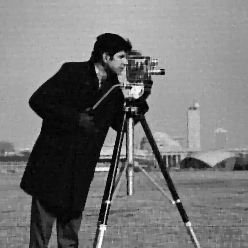}%
		\label{fig:cameraman_deblur_as_iddbm3d}}
	\caption{Use of the IDD-BM3D deblurring method and Framework 1 in a blurred image with unknown boundaries. The experiments were run for the $256 \times 256$ \textit{cameraman} image blurred with a $9 \times 9$ boxcar filter. (a) Plain IDD-BM3D (ISNR = -14.62 dB). (b) IDD-BM3D with pre-smooth\-ing of the blurred image's borders (ISNR = 8.27 dB). (c) Framework 1, with step 2 implemented through IDD-BM3D (ISNR = 9.64 dB). (d) ADMM-based implementation of Framework 1 (Proposed-AD version, $\lambda = 4 \times 10^{-5}$ ) (ISNR = 8.43 dB).}
	\label{fig:cameraman_bm3d}
	\vspace{-15pt}	
\end{figure*}

\subsection{Proposed method: deblurring with unknown boundaries}
\label{sec:deblurring}

In this and the following subsections, we present results of the ADMM-based deblurring method proposed in Section~\ref{sec:partialadmm}, and we compare it with two published state-of-the art methods. The first of these is AM, which was chosen because it is an ADMM-based method specifically developed for the problem of deblurring with unobserved pixels, and therefore bears some resemblance to the proposed partial ADMM. Although the MATLAB code for AM was made available by the authors of~\cite{Almeida2013a}, we implemented the method ourselves in a form that was more efficient than the published one, by performing the products by the matrices $\D$ and $\D^T$ in the space domain instead of the frequency domain (the latter required the computation of FFTs and element-wise multiplications, while our implementation only required subtractions). The second method used for comparison was the primal-dual algorithm of Condat~\cite{Condat2013}, which we shall denote by CM. We chose this method because it can be expressed in a form that does not require the inversion of matrices related to the blurring operator, and this inversion is a computational bottleneck of most deblurring methods based on exact blurring models, as discussed in Section~\mbox{\ref{sec:problemstatement}}. While~\cite{Condat2013} does not explicitly consider the problem of image deblurring with unobserved pixels, in~\cite{Condat2014} it is shown how to adapt the method to a deblurring and demosaicing problem. We implemented this adaptation of the method in a form appropriate to the observation model of Eq.\ \eqref{model3}. For completeness, we also show results obtained with an approximation of the standard ADMM. As discussed in Section \ref{sec:partialadmm}, the direct application of that method is impracticable, even for small images, in most present-day computers. We give the results obtained by approximately solving Eq.~\eqref{eq:admm_xz_1} through 1000 conjugate-gradients iterations. This number of iterations was experimentally chosen to approximately yield the best speed. We denote this method by `ADMM-CG'.

All the methods under comparison minimize equivalent cost functions, and therefore we considered it quite probable that they would yield virtually identical estimated images, even though, in our tests, we were not strictly under the convergence conditions of Theorem~\ref{th:main}, for reasons that we will discuss further on. As described ahead, we experimentally confirmed, in a few cases, that all the methods converged to essentially the same images, and therefore we focused the subsequent comparisons on the computational efficiencies of the various methods, and not on the quality of their results. The computation times that we report ahead were obtained using MATLAB on an Intel Core i7 CPU running at 3.20~GHz, with 32~GB of RAM.

We used, in the deblurring tests, two well-known images: \textit{Lena} (in grayscale) and \textit{cameraman}, both of size $256 \times 256$ pixels. These images were blurred with boxcar filters of sizes between $3 \times 3$ and $21 \times 21$, and with truncated Gaussian filters with supports of the same sizes; for a Gaussian filter with a support of size $l \times l$, we used a standard deviation of $\sqrt{l}$.

As in~\cite{Almeida2013a, Matakos2013, Condat2014}, we used isotropic TV to regularize our problem, i.e., we made $\phi(\D \x) = \lambda \sum_{j = 1}^{k + d} \sqrt{(\D_h \x)_j^2 + (\D_v \x)_j^2}$, where $(\mathbf{a})_{j}$ denotes the $j$th element of the vector $\mathbf{a}$; $\D_h$ and $\D_v$ were such that the products by these matrices computed, respectively, the horizontal and vertical first-order differences of a discrete image, with periodic boundaries; $\D = [\D_h^T \D_v^T]^T$; and $\lambda > 0$. The matrix $\D^T \D$ was BCCB, as required for both the proposed method and AM to be efficient (CM has no such requirement).

With the isotropic TV regularizer, the solution of problem~(\ref{eq:admm_aprox_3}) is obtained by a vector soft-thresholding operation, $\vs_x^{i+1} = \text{vector-soft} \big( \vs_x^i, \frac{\lambda}{\mu} \big)$ \cite{Wang2008, Donoho1995}. With this regularizer, we can guarantee the uniqueness of the solution of problem \eqref{eq:problem_xz} if $\mathcal{N}(\Hm) \cap \mathcal{N}(\D) = \{\mathbf 0\}$, which is true if $\textbf{1}_k \notin \mathcal{N}(\Hm)$, where $\mathcal{N}(\cdot)$ denotes the null space of a matrix. The latter condition is normally verified, since real-world blur kernels usually have nonzero DC gain. 

The guarantee of convergence given by Theorem \ref{th:main} requires that matrix $\D$ be full column rank. An assumption of this kind is common in the literature, when studying the convergence of primal-dual methods, even though it can be relaxed in some cases (see, e.g.,~\cite{Boyd2011} for the ADMM case). With the isotropic TV regularizer, $\D$ is not full column rank (its rank is $k+d-1$), which means that we did not have a formal guarantee of convergence of the proposed method in our experimental tests. This was not the case for AM, where the use of an extra variable splitting makes $\D$ full rank. It was also not the case for CM, in which $\D$ does not need to be full rank. In practice, we found that the four methods always converged, and we had some experimental evidence that they all converged to the same solution, as reported ahead. If we wished to have a formal guarantee of convergence of the proposed method, we could have used one of two approaches: (1) to use a new variable to decouple the convolution operator from $\x$, which would lead to, e.g., $\D = [\D_h^T \D_v^T \Hm^T]^T$, which is full column rank if $\mathcal{N}(\D_h) \cap \mathcal{N}(\D_v) \cap \mathcal{N}(\Hm) = \{\mathbf{0}\}$ (the regularizer would need to be changed accordingly), or (2) to modify the discrete difference matrices, making them full rank, e.g., by adding to them $\varepsilon \I_{k+d}$ with a small $\varepsilon > 0$.

In the proposed method, the complexity of an iteration was dominated by FFTs: three for steps \eqref{eq:admm_aprox_1} and \eqref{eq:admm_aprox_2}, and one for step~(\ref{eq:admm_aprox_3}); the number of FFTs thus depended on the number of BGS passes (each pass involved three FFTs). In our implementation, the other two methods both involved four FFTs per iteration.\footnote{The publicly available implementation of AM requires seven FFTs per iteration.} Our method required the storage of one variable with dimension $d$ and five with dimension ${k+d}$. AM required the storage of seven variables with dimension ${k+d}$, and CM required the storage of four with dimension ${k+d}$.

In order to ensure fair comparisons, the initial estimate of $\x$ was obtained in the same way for the four methods, by padding the observed image with a boundary zone formed by repeating the border pixels of the observed image. This initialization led to slightly faster results than other alternatives, such as linear interpolation between opposite-border pixels.

All the methods have parameters that need to be adjusted. The optimal tuning of the regularization parameter $\lambda$ is a complex issue, and a number of techniques have been proposed for that purpose, e.g. using Stein's unbiased risk estimator~\cite{Donoho1995}, generalized cross-validation~\cite{Golub1979}, or a measure of the whiteness of the reconstruction error~\cite{Almeida2013}. Since, as mentioned above, we wanted to focus our comparison on the computational costs, and not on the quality of the deblurring results, we fixed $\lambda = 5 \times 10^{-6}$, a value that yielded visually good estimated images.

Another parameter that needed to be set, both in our method, in AM and in ADMM-CG, was the penalization parameter $\mu$, which influences the convergence speed. The authors of~\cite{Almeida2013a} and of~\cite{Matakos2013} developed their own heuristics to set this parameter. Since heuristics like these require an amount of fine-tuning that can be time-consuming, and since the tuning done for some images may not extend to other images, we instead followed a simple strategy, proposed in~\cite{Boyd2011}. It involves little computation, but requires that we keep track of both $\vs^{i+1}$ and $\vs^i$. The penalization parameter is adapted in every iteration, according to 
\begin{equation} \label{eq:varying_penalty}
	\mu^{i+1} = \left\{
	\begin{array}{l l}
		2\mu^i & \quad \text{if $\| \dv^{i+1} - \dv^i \| > t \| \vs^{i+1} - \vs^i \|$,}\\
		0.5\mu^i & \quad \text{if $\| \vs^{i+1} - \vs^i \| > t \| \dv^{i+1} - \dv^i \|$,}\\
		\mu^i & \quad \text{otherwise.}
	\end{array} \right.
\end{equation}
Additionally, $\dv^{i+1}$ needs to be rescaled after updating $\mu$: it should be halved in the first case above, and doubled in the second one. In the proposed method, we used $t=3$. We used the same strategy to set the penalty parameters of AM. There are two such parameters in that method: one similar to the one of our method, and another one relating to the variable splitting used to decouple matrix $\M$ from matrix $\tilde{\T}$. We adapted the two parameters independently, and used $t=10$ for both because lower values led to instabilities. For ADMM-CG, we used $t=3$. The use of the same strategy to set the penalty parameters of the three methods was intended to make the comparison between both as fair as possible. The convergence proof of Theorem 1 is difficult to adapt to a varying $\mu$, but still applies if this parameter is kept fixed after a certain number of iterations. \footnote{For alternative strategies for setting the penalization parameters, see~\mbox{\cite{Boyd2011} and \cite{He2000}}.}

As already mentioned, another parameter of our method is the number of BGS passes. The proof in Appendix~\ref{sec:proof} assumes that a single pass is used, and is easily extendable to any fixed number of passes. If the number of passes changes along the iterations, the proof still applies if that number becomes fixed after a certain number of iterations. For setting this parameter, we found that a simple strategy was useful: for a blur of size $(2b+1)\times(2b+1)$, we set the number of passes to $b$. We found this strategy to yield a good choice of the number of passes, irrespective of the noise level of the observed image. In what follows, we report the results obtained using this adaptive strategy (designated by `Proposed-AD'), and also the results obtained with just one BGS pass (designated by Proposed-1'); the latter are reported because they correspond to an especially simple, and therefore interesting situation.

In CM, in order to guarantee convergence, we chose its proximal parameters $\tau$ and $\sigma$ so as to satisfy the condition given in~\cite{Condat2014}, $\tau( \frac{\beta}{2} + \sigma \| \D^T \D \|_I ) < 1$, where $\| \cdot \|_I$ denotes the induced norm of a matrix, and $\beta$ is the Lipschitz constant of the gradient of $f$, $\beta = \| \M_y \Hm\|_I^2$. It can be shown~\cite{Chambolle2004} that $\| \D^T \D \|_I \leq 8$; we also have $\| \M_y \Hm\|_I \leq \| \M_y \|_I \| \Hm\|_I$, and $\| \M_y \|_I = 1$, $\| \Hm \|_I = 1$.  Therefore, in order to obey the convergence condition, we chose $\tau = \frac{0.99}{0.5+8 \sigma}$. We verified experimentally that the speed of convergence was approximately maximal as long as $\sigma \in [10^{-8}, 10^{-4}]$, and therefore chose $\sigma = 10^{-6}$.

The variable $\z$, which exists only in the proposed method and in ADMM-CG, was initialized with an approximate solution of the  quadratic regularization problem
\begin{equation*}
	 \min_{\x,\z} f(\x,\z) + \frac{\tau}{2}\|\x \|^2,
\end{equation*}
where $\tau$ is a small positive value.  The approximate solution was computed by alternating minimization with respect to  $\x$ and  $\z$, yielding the iterative procedure
\begin{align}
   \label{eq:ini_x}
   \x^i & = \left( \Hm^T\Hm + \tau{\bf I}\right)^{-1}\Hm^T
                        \begin{bmatrix}
                           \y \\
                           \z^{i-1}
                         \end{bmatrix}\\
  \label{eq:ini_z}
  \z^i & = {\bf M}_z\Hm \x^i,
\end{align}
for $i=1,2,\cdots,100$ , $\tau =10^{-3}$, with $\z^0$ obtained by padding $\bf y$ with its border elements.  The complexity involved in the computation of  $\x^{i+1},\z^{i+1}$ is dominated by the solution of the linear system implicit in \eqref{eq:ini_x}. Given that  $\left( \Hm^T\Hm + \tau{\bf I}\right)$ is BCCB, that solution can be efficiently found in the frequency domain, with complexity $O[(k+d) \log (k+d)]$.

To check whether the four methods converged to the same result, we ran them for a very large number of iterations ($10^6$) in a few of the cases mentioned ahead, and found that, in each case, the results were essentially the same for all methods: the root-mean-square errors (RMSE)\footnote{The RMSE was defined as

\vspace{.5mm}

{\centering $RMSE =  \sqrt{1 / (m' \times n')\sum_{i=1}^{m' \times n'} (x_i - x_i^r)^2}$, \par}

\noindent where $x_i$ and $x_i^r$ denote, respectively, the pixels of the estimated image and of the reference image.} among the results from different methods were always below $10^{-7}$, which is much below what is visually distinguishable. Given this, we arbitrarily chose, for all tests, the results of one of the methods (AM), after the mentioned $10^6$ iterations, as representatives of the solutions of the corresponding problems. We used these representatives as reference images for the evaluation of the quality of the results of the four methods. Our choice of AM to compute the reference images did not especially benefit this method relative to the other ones, because, if we had chosen any of the other two methods instead of AM, the RMSE values of the results of the tests, computed relative to them, would have been essentially the same. We did not use the original sharp images as references, for two reasons. First, given the experiments mentioned in the beginning of this paragraph, we had good reasons to believe that the four methods converged to the same fixed point, in each problem. Second, the solutions of the optimization problems were slightly different from the original sharp images, due to the presence of the noise and of the regularizer, and our main interest was in assessing the speed of convergence of the methods to the solutions of the optimization problems, not in the recovery of the original images. If we had used the original images as references, it might happen that, in their path to convergence, some of the methods would pass closer to those images than other methods, and this could make them stop earlier. This would give them an apparently better speed, while they might not truly have a better convergence speed. In what follows, when we mention RMSE values, these were computed relative to the above-mentioned reference images, using the whole images, including the boundary zone, in the computation.


Fig.~\ref{fig:grafico_conv} illustrates the behavior of the four methods during the optimization. The estimated images from all the methods were visually indistinguishable from one another. Figure~\ref{fig:cameraman} shows a result of the Proposed-AD method. In Table~\ref{tab:results}, we show the computing times for \emph{cameraman} with Gaussian blurs of various sizes. The times for \textit{Lena} with the same blurs, and for both images with the boxcar blurs, are presented in Appendix~\ref{sec:app_exp}. In summary, the Proposed-AD method was faster than AM for small and medium-sized boxcar blurs and for small Gaussian blurs, the two methods had similar speeds for large boxcar blurs and medium-sized Gaussian blurs, and AM was faster for large Gaussian blurs.


\begin{figure}[!t]
	\centering
	\vspace{-6pt}
	\includegraphics[trim={0 1pt 0 14pt},clip,scale=.48]{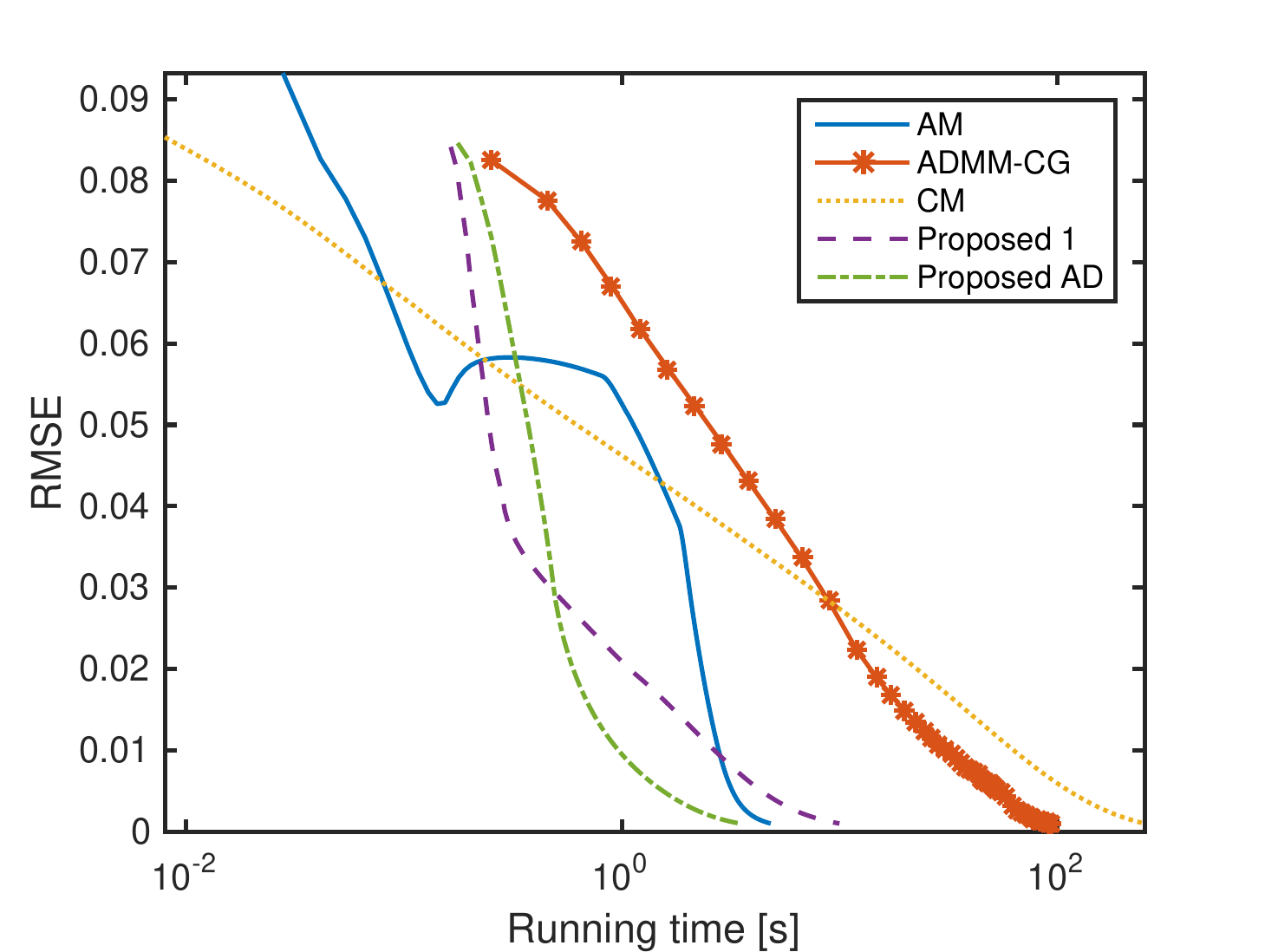}%
	\vspace{-6pt}
	\caption{Deblurring: RMSE of the estimated images as a function of running time, for the various methods. This test used the \textit{cameraman} image with a $13 \times 13$ boxcar blurring filter.}	
	\label{fig:grafico_conv}
\end{figure}

\begin{figure}[!t]
		\vspace{-10pt}
	\centering
	\subfloat[]{\includegraphics[scale=.375]{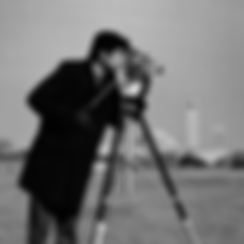}%
		\label{fig:cameraman_original}}
	\hfil
	\subfloat[]{\includegraphics[scale=.375]{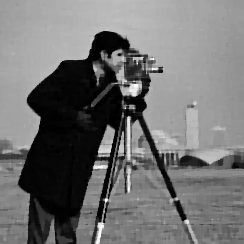}%
		\label{fig:cameraman_deblur}}
	\caption{Deblurring: observed (a) and estimated (b) images using the Proposed-AD method. The experiments were run for the \textit{cameraman} image with a $13 \times 13$ boxcar blur.}
	\label{fig:cameraman}
		\vspace{-15pt}
\end{figure}

\begin{table}
	\renewcommand{\arraystretch}{1}
	\caption{Results for \textit{cameraman} with Gaussian blurs of various sizes. The value of $\kappa$ is the condition number of matrix $\Hm$. The two rightmost columns show the number of iterations and the time taken to reach an RMSE below $10^{-3}$.}
	\label{tab:results}
	\centering
\begin{tabular}{l|c|c|c|c}
Method &Blur size &$\kappa \times 10^3$ &Iterations &Time (s) \\
\hline
Proposed-1 & \multirow{5}{*}{5} & \multirow{5}{*}{247.0} &   85 & 0.866  \\ 
Proposed-AD & & &   56 & 0.770 \\ 
AM & & &  178 & 1.973 \\ 
ADMM-CG & & &   45 & 29.355 \\ 
CM & & & 14562 & 63.797 \\ 
\hline 
Proposed-1 & \multirow{5}{*}{13} & \multirow{5}{*}{21590.6} &  531 & 4.590  \\ 
Proposed-AD & & &  117 & 2.687 \\ 
AM & & &  313 & 3.399 \\ 
ADMM-CG & & &   71 & 106.450 \\ 
CM & & & 123912 & 539.466 \\ 
\hline 
Proposed-1 & \multirow{5}{*}{21} & \multirow{5}{*}{686384.8} & 5366 & 45.169  \\ 
Proposed-AD & & &  542 & 17.858 \\ 
AM & & &  882 & 9.478 \\ 
ADMM-CG & & &  319 & 223.828 \\ 
CM & & & 305227 & 1334.286 \\ 
\hline 
\end{tabular}
\end{table}

\subsection{Proposed method: inpainting, superresolution, and demosaicing} \label{sec:exp_app}

The experiments reported in this subsection used the settings that were described in the previous subsection, unless otherwise noted. In all experiments, the boundary zone was considered as unobserved, as in the previous subsection.

\paragraph{Inpainting}

As mentioned in Section \ref{sec:framework}, the inpainting problem can be formulated within the proposed framework by simply including in $\z$ the unobserved pixels. We performed tests on a problem of simultaneous deblurring and inpainting, using the \emph{Lena} image with a boxcar blur of size $13 \times 13$, in which some zones were unobserved, as shown in Fig.~\ref{fig:lena_inpainting} (a). The regularization parameter was kept at $\lambda = 2 \times 10^{-6}$. The result obtained with the Proposed-AD method can be seen in Fig.~\ref{fig:lena_inpainting} (b). The other methods produced visually indistinguishable results. The evolution of the different methods during the optimization process is shown in Fig.~\ref{fig:grafico_conv_inpainting}.

\begin{figure}[!t]
	\centering
	\subfloat[]{\includegraphics[scale=.5]{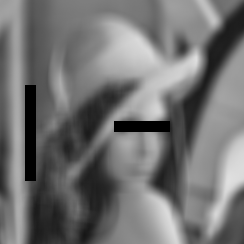}%
		\label{fig:lena_original_inpainting}}
	\hfil
	\subfloat[]{\includegraphics[scale=.5]{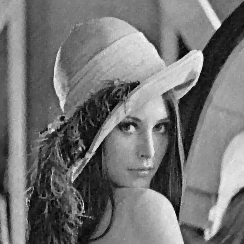}%
		\label{fig:lena_deblur2_inpainting}}
	\caption{Inpainting with deblurring: observed (a) and estimated (b) image using the Proposed-AD method. In the observed image, the black rectangles correspond to unobserved regions.}
	\label{fig:lena_inpainting}
\end{figure}

\begin{figure}[!t]
	\centering
	\includegraphics[trim={0 0 0 20pt},clip,scale=.48]{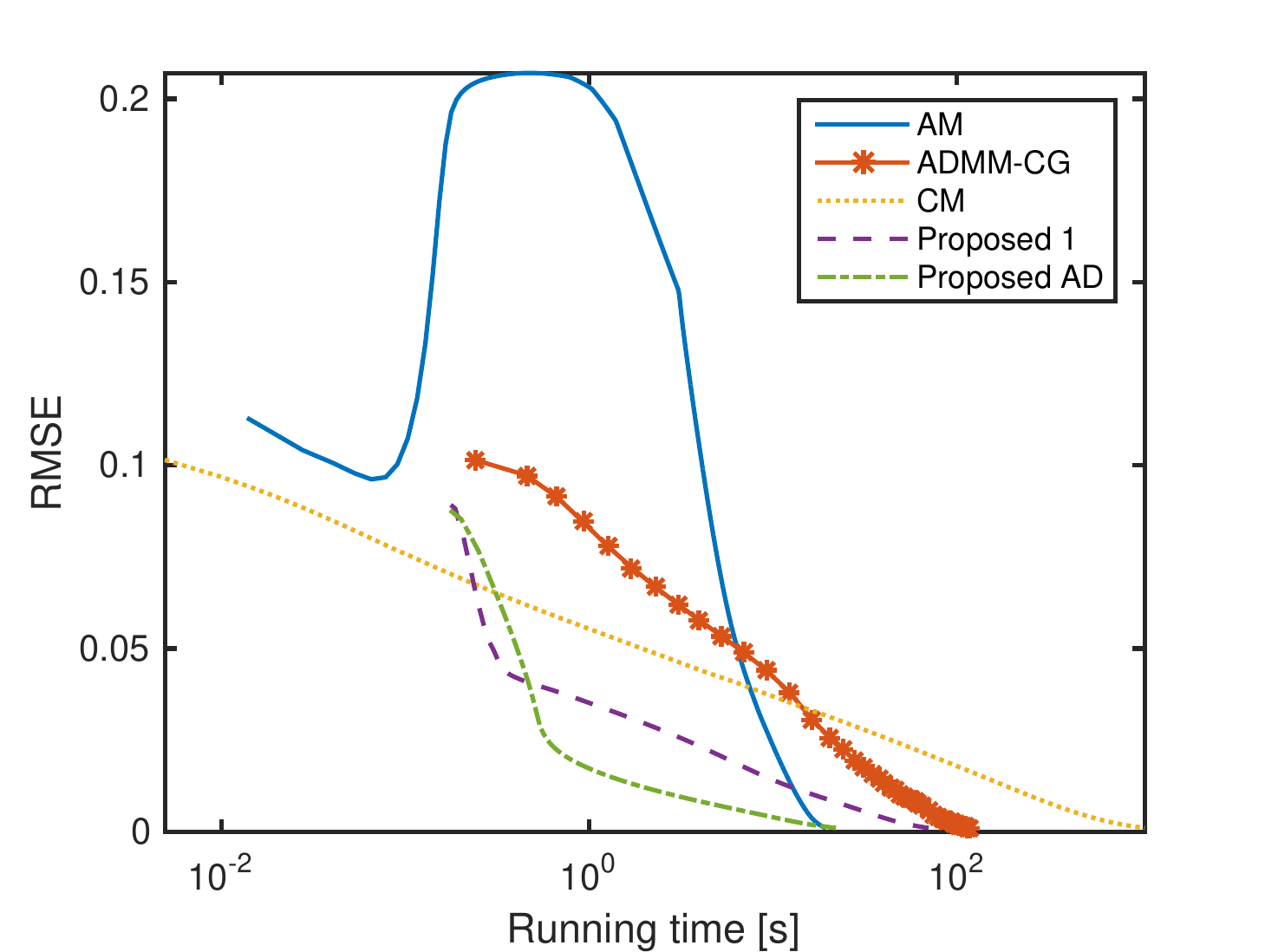}%
	\caption{Inpainting with deblurring: RMSE of the estimated images as a function of running time, for the various methods.}	
	\label{fig:grafico_conv_inpainting}
\end{figure}

\paragraph{Superresolution} In superresolution problems, $\z$ should include the extra pixels that would have been present if the image had the higher resolution. For this test, we used a $608 \times 337$ crop of a single band of the so-called \emph{Pavia} hyperspectral image. The experiments consisted of increasing the resolution by a factor of 3, and we used a $3 \times 3$ boxcar blurring filter in the estimation of the high-resolution image. The regularization parameter was set, again, to $\lambda = 2 \times 10^{-4}$. Figs.~\ref{fig:pavia_sr} and~\ref{fig:grafico_conv_sr} show the results. The estimated images from all the methods were visually indistinguishable from one another.

\begin{figure}[!t]
	\centering
	\subfloat[]{\includegraphics[scale=0.75]{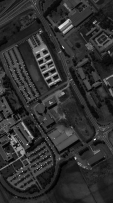}%
		\label{fig:pavia_original_sr}}
	\hfil
	\subfloat[]{\includegraphics[scale=.25]{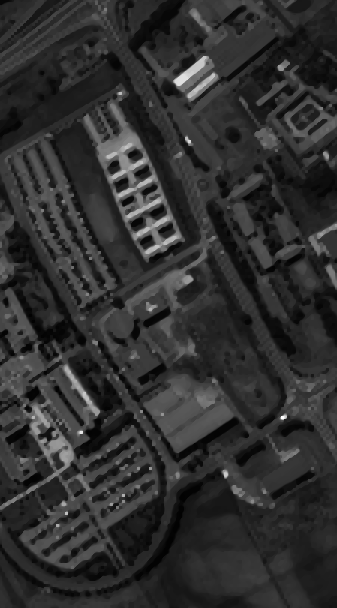}%
		\label{fig:pavia_deblur2_sr}}
	\caption{Superresolution: observed (a) and estimated (b) images using the Proposed-AD method. The upsampling ratio was 3.}
	\label{fig:pavia_sr}
\end{figure}

\begin{figure}[!t]
	\centering
	\includegraphics[trim={0 0 0 22pt},clip,scale=.48]{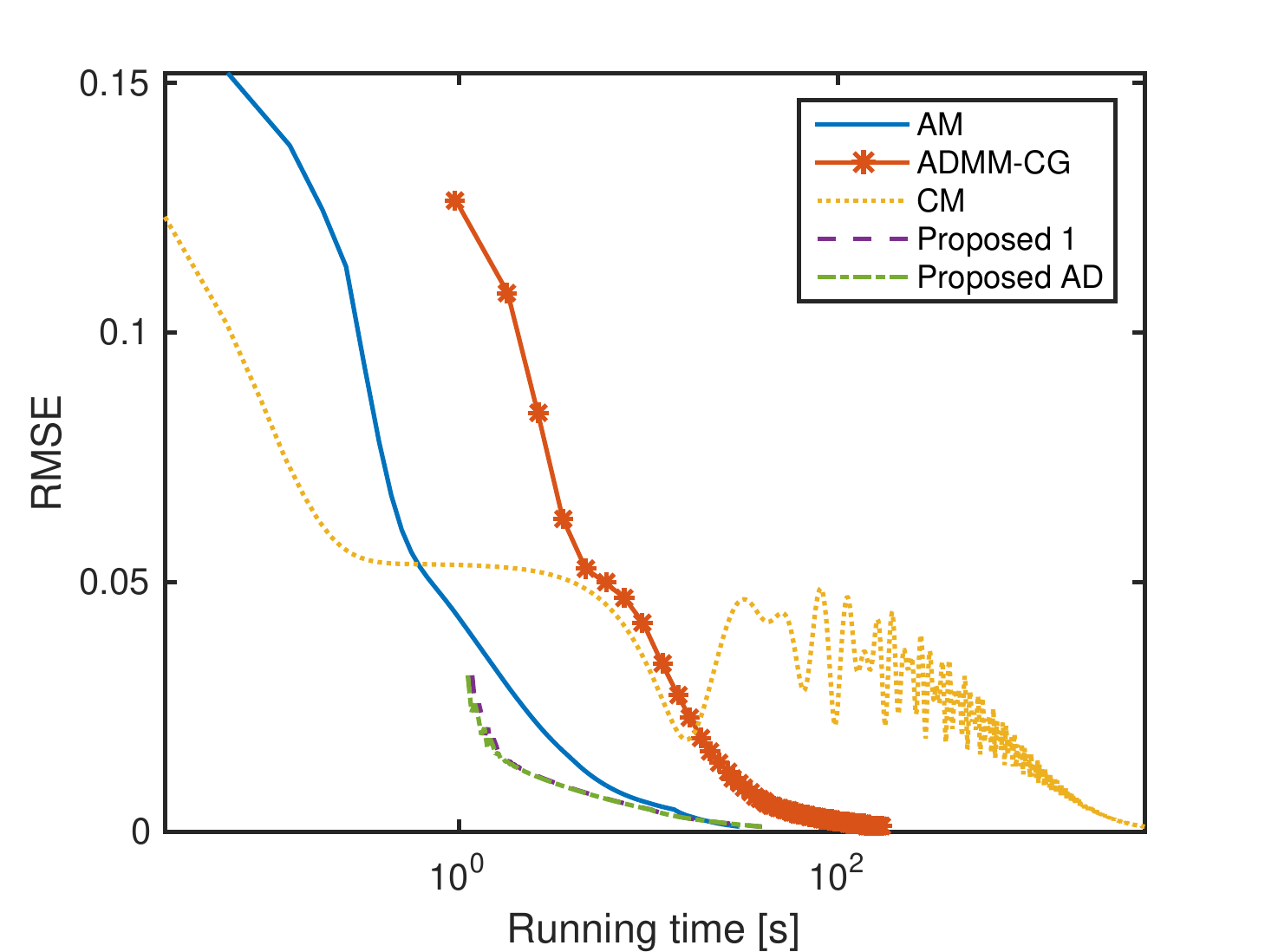}%
	\caption{Superresolution: RMSE of the estimated images as a function of running time, for the various methods.}	
	\label{fig:grafico_conv_sr}
\end{figure}

\paragraph{Demosaicing} \looseness -1 The demosaicing operation, which is a step of the processing pipeline used in most digital cameras, seeks to reconstruct a full-resolution color image from its subsampled RGB channels. In most digital cameras, the image sensor observes only one color in each pixel, and therefore the color channels are subsampled relative to the resolution of the sensor. The goal of demosaicing is to fuse the color channels, obtaining full-resolution images for all of them. We performed tests on a problem of simultaneous deblurring and demosaicing. We used a $289 \times 253$ crop of the well-known \emph{parrot} image as ground truth, and synthesized the observed image by first blurring it with an 8x8 boxcar filter, and then applying a Bayer color filter. In the representation of the problem within our framework, vector $\y$ contained the observed pixels from the three color channels, and vector $\z$ contained the unobserved pixels from all of them (including the boundary zone). In the vector $\y$, the observed pixels were grouped according to the channels they belonged to, by making $\y = [(\y^1)^T (\y^2)^T (\y^3)^T]^T$; the same was done for the vectors $\z$ and $\x$. The data-fitting term had the form $f(\uu) = 1/2 \sum_{j=1}^3 \big\| \big[ \begin{smallmatrix} \y^j \\ \z^j \end{smallmatrix} \big] - \Hm \x^j \big\|^2$. As in the demosaicing example of~\cite{Condat2014}, we used a vectorial version of the TV regularizer, which promotes the alignment of the discontinuities across channels: $\phi(\D \x) = \lambda \sum_{i=1}^{k+d} \sqrt{\sum_{j=1}^3 \{ (\D_h \x^j)_i^2 + (\D_v \x^j)_i^2 \}}$. The regularization parameter was set to $\lambda = 10^{-5}$. The results can be seen in Figs.~\ref{fig:parrot_demosaicing} and~\ref{fig:grafico_conv_demosaicing}. Once again, the estimated images from all the methods were visually indistinguishable from one another.

\begin{figure}[!t]
	\centering
	\subfloat[]{\includegraphics[scale=.3]{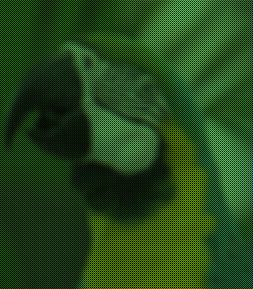}%
		\label{fig:parrot_original_demosaicing}}
	\hfil
	\subfloat[]{\includegraphics[scale=.3]{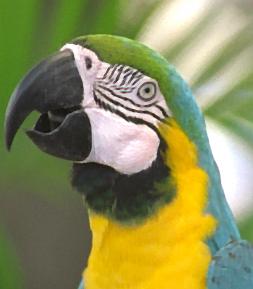}%
		\label{fig:parrot_deblur2_demosaicing}}
	\caption{Demosaicing: observed (a) and estimated (b) image using the Proposed-AD method.}
	\label{fig:parrot_demosaicing}
\end{figure}

\begin{figure}[!t]
	\centering
	\includegraphics[trim={0 0 0 20pt},clip,scale=.48]{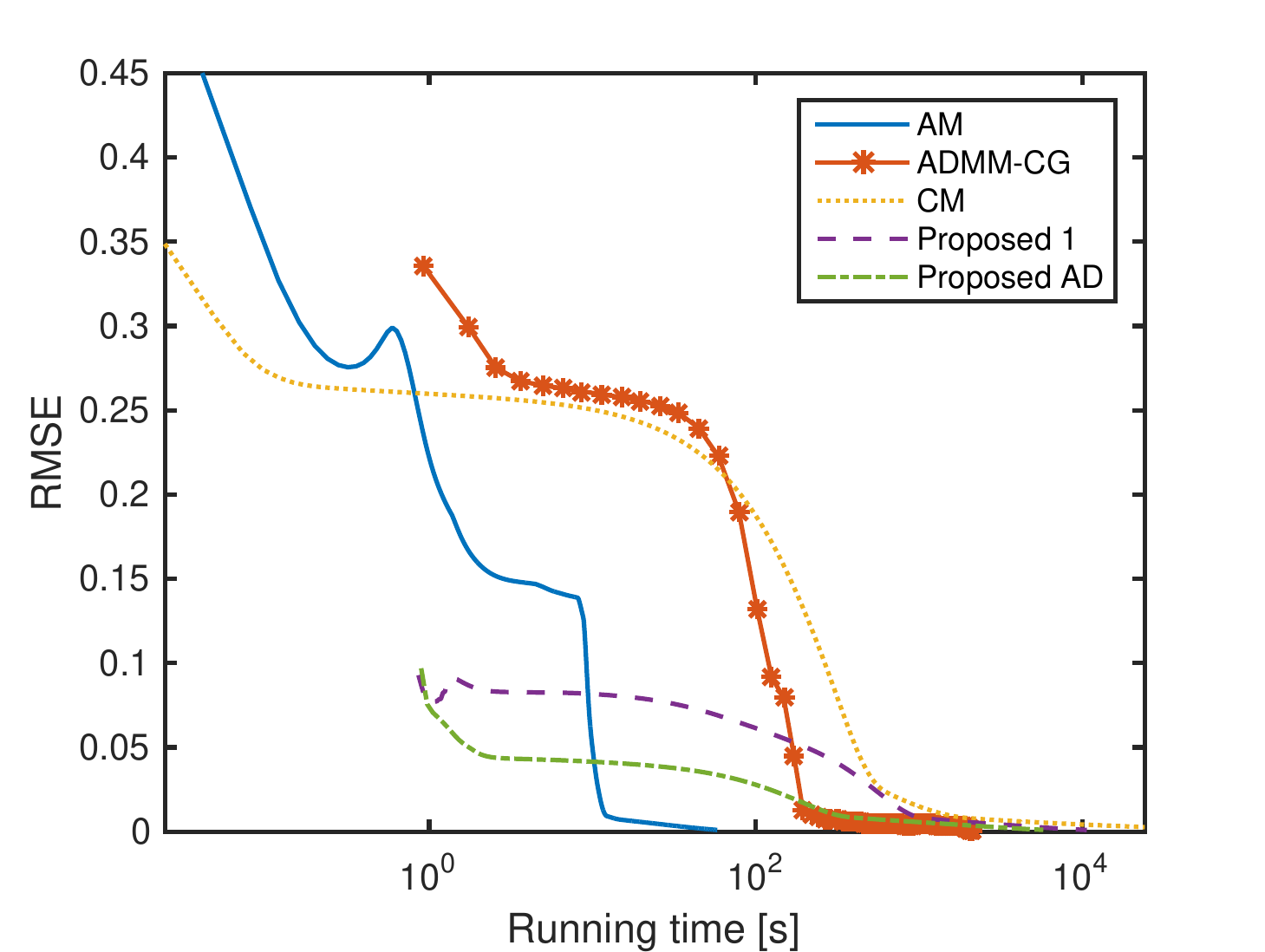}%
	\caption{Demosaicing: RMSE of the estimated images as a function of running time, for the various methods.}	
	\label{fig:grafico_conv_demosaicing}
\end{figure}

\subsection{Appraisal}
\label{sec:appraisal}

The example of the use of IDD-BM3D within the proposed framework illustrates the fact that this framework can be used to convert existing deblurring methods that assume artificial boundary conditions to methods that use unknown boundaries. As far as we know, there is no other published way to accomplish this. This is a simple alternative to the use of edge tapering, yielding results of better quality, as illustrated in Subsection~\ref{iddbm3d}.

The proposed partial ADMM implementation of Framework 1 (more specifically, the Proposed-AD form) showed a performance similar to the performance of the state-of-the-art AM method, both in terms of final results and of computational speed. CM yielded similar deblurred images, but showed a significantly lower computational performance. We conjecture that this was due to the fact that both the proposed method and AM, being based on ADMM, make use of second-order information on the objective function (in terms of the matrix $\Hm^T \Hm + \mu \D^T \D$). CM, on the other hand, does not use any information of this kind. This difference in speed agrees with our experience in other image enhancement problems, in which we have repeatedly found CM-like methods to be significantly slower than ADMM-based ones.

\section{Conclusion} \label{sec:conclusion}

We have proposed a framework for deblurring images with unobserved pixels. This framework can be used to convert most deblurring methods to unknown boundaries, irrespective of the specific boundary conditions that those methods assume, being a simple, high-quality alternative to the use of edge tapering.

\looseness -1 An ADMM-based deblurring method that falls within the mentioned framework was proposed, and a proof of convergence was provided. This method can be seen as a partial ADMM with a non-dualized variable. Experimental results on problems of deconvolution, inpainting, superresolution and demosaicing, with unknown boundaries, showed, for the proposed method, a performance at the level of the state of the art.

\appendices

\section{ADMM} \label{sec:app_admm}

ADMM can be used to solve convex optimization problems of the form
\begin{equation} \label{eq:admm}
\underset{\uu}{\text{minimize}} \quad f(\uu) + \psi(\K \uu),
\end{equation}
where $f: \mathbb{R}^{n} \rightarrow \mathbb{R} \cup \{+\infty\}$ and $\psi: \mathbb{R}^{m} \rightarrow \mathbb{R} \cup \{+\infty\}$ are closed proper convex functions, and $\K \in \mathbb{R}^{m \times n}$.
The problem is first rewritten as
\begin{equation}
\begin{aligned}
& \underset{\uu, \vs}{\text{minimize}}
& & f(\uu) + \psi(\vs)\\
& \text{subject to}
& & \vs = \K \uu,
\end{aligned}
\end{equation}
where $\vs \in \mathbb{R}^{m}$.
This problem has the augmented Lagrangian~\cite{Nocedal2006}
\begin{equation} \label{eq:al}
\mathcal{L}(\uu, \vs, \dv) = f(\uu) + \psi(\vs) + \frac{\mu}{2} \Big\| \vs - \K \uu -  \dv \Big\|^2,
\end{equation}
where $\mu>0$ is a penalization parameter, and $\dv \in \mathbb{R}^{m}$ is the so-called scaled dual variable ($\mu \dv$ is the dual variable). 

The standard ADMM solves problem~(\ref{eq:admm}) through an iteration on a set of problems that are often much simpler,
\begin{align}
\uu^{i+1}\ &= \underset{\uu}\argmin \quad \mathcal{L}(\uu, \vs^i, \dv^i), \label{eq:admm_1} \\
\vs^{i+1} &= \underset{\vs}\argmin \quad \mathcal{L}(\uu^{i+1}, \vs, \dv^i), \label{eq:admm_2} \\
\dv^{i+1} &= \dv^i - (\vs^{i+1} - \K \uu^{i+1}). \label{eq:admm_3}
\end{align}

The convergence of ADMM was established in~\cite{Eckstein1992} (among others) by proving a more general version of the following theorem:

\begin{theorem} \label{th:admm}
Consider a problem of the form~\eqref{eq:admm}, in which $\K$ has full column rank. Let $\mathbf{d}^0 \in \mathbb{R}^{n}$ and  $\mu > 0$.
Assume that $\{\uu^i\}$, $\{\vs^i\}$, and  $\{\dv^i\}$ conform, for all $i$, to
\begin{align}
&\uu^{i+1} = \underset{\uu}\argmin \mathcal{L}(\uu, \vs^i, \dv^i), \\
&\vs^{i+1} = \underset{\vs}\argmin \mathcal{L}(\uu^{i+1}, \vs, \dv^i), \\
&\dv^{i+1} = \dv^i - (\vs^{i+1} - \K \uu^{i+1}).
\end{align}
If problem~\eqref{eq:admm} has a Karush-Kuhn-Tucker (KKT) pair,\footnote{A pair $(\uu$, $\pp)$ is said to be a Karush-Kuhn-Tucker pair of problem~\eqref{eq:admm} if $- \K^T\pp \in \partial f(\uu)$ and $\pp \in \partial \psi (\K \uu)$.} then the sequence $\{\uu^i\}$ converges to a solution of that problem, $\{\mu \dv^i\}$ converges to a solution of the dual problem, and $\{\vs^i\}$ converges to $\K \uu^*$, where $\uu^*$ is the limit of $\{\uu^i\}$. If problem~\eqref{eq:admm} has no KKT pair, then at least one of the sequences $\{\dv^i\}$ or $\{\vs^i\}$ is unbounded.
\end{theorem}

\section{Proof of Theorem 1} \label{sec:proof}

The proof follows a similar machinery to the ones proposed in~\cite{Chen1994, He2002, He2012a, Deng2012, Yang2011}. We start by deriving some preliminary results. Consider problem~\eqref{eq:problem_xz} with $f$ given by~\eqref{eq:quadratic_f}, and define, as in the main text,
\begin{equation*}
\uu = \begin{bmatrix}
\x \\
\z
\end{bmatrix},
\quad\K = \begin{bmatrix}
\D & \mathbf{0} \\
\mathbf{0} & \I_{d}
\end{bmatrix},
\quad
\psi (\K \uu ) = \phi(\D \x).
\end{equation*}
Also define 
\begin{equation*}
\bm{\omega} = \begin{bmatrix}\uu \\ \vs \\ \dv \end{bmatrix},
\quad \wt = \begin{bmatrix}\z \\ \vs_x \\ \dv_x \end{bmatrix},
\quad \G = \Bigl[ \begin{smallmatrix}
\C & \mathbf{0} & \mathbf{0}\\ \mathbf{0} &\mu \mathbf{I}_l & \mathbf{0}\\ \mathbf{0} & \mathbf{0}&\mu \mathbf{I}_l
\end{smallmatrix} \Bigl],
\end{equation*}
\begin{equation*}
\bm{\Pi} = \Bigl[ \begin{smallmatrix}
\mathbf{0} & \mathbf{I}_d & \mathbf{0} & \mathbf{I}_d & \mathbf{0} & \mathbf{0}\\ 
(\D^T \D)^{-1} \D^T & \mathbf{0} & \mathbf{I}_l & \mathbf{0} & \mathbf{0} & \mathbf{0}\\ 
\mathbf{0}  & \mathbf{0} & \mathbf{0} & \mathbf{0} & \mathbf{I}_l & \mathbf{0} \end{smallmatrix} \Bigl]^T,
\quad \bm{\Gamma} = \Bigl[ \begin{smallmatrix}
\mathbf{0} & \mathbf{I}_d & \mathbf{0} & \mathbf{0} & \mathbf{0} & \mathbf{0}\\ 
\mathbf{0} & \mathbf{0} & \mathbf{I}_l & \mathbf{0} & \mathbf{0} & \mathbf{0}\\ 
\mathbf{0} & \mathbf{0} & \mathbf{0} & \mathbf{0} & \mathbf{I}_l & \mathbf{0} \end{smallmatrix} \Bigl].
\end{equation*}
Note that $\wt = \bm{\Gamma} \bm{\omega}$, and therefore $\bm{\omega}$ univocally determines $\wt$.

Lemmas \ref{lemma1}--\ref{lemma5} present several results concerning the partial ADMM~\eqref{eq:admm_aprox_1}--\eqref{eq:admm_aprox_4} with $f$ defined as in~\eqref{eq:quadratic_f}. 

\begin{lemma} \label{lemma1}
	Assume that $\phi$ is closed proper convex and coercive and that $\D$ is full column rank. Then, the set of solutions of problem~(\ref{eq:problem_xz}) is non-empty.
\end{lemma}

\begin{proof}
	This result follows immediately from~\cite[Proposition 11.14]{Bauschke2011}.
\end{proof}

\begin{lemma} \label{lemma2}
	
	Under the assumptions of Lemma~\ref{lemma1}, let $\wt^i = [(\z^i)^T (\vs_x^i)^T (\dv_x^i)^T]^T$ denote the vector $\wt$ that results from the $i$th iteration of the partial ADMM~\eqref{eq:admm_aprox_1}--\eqref{eq:admm_aprox_4}, obtained with a given initialization of $\x$, $\z$, $\vs_x$, and $\dv_x$. Denote with an asterisk, as in $\bm{\omega}^* = [(\uu^*)^T (\vs^*)^T (\dv^*)^T]^T$, a KKT pair of problem~(\ref{eq:problem_xz}).\footnote{We consider problem \eqref{eq:problem_xz} reformulated as \eqref{eq:problem_u}, and say that a pair $(\uu$, $\pp)$ is a Karush-Kuhn-Tucker pair of that problem if $- \K^T\pp \in \partial \bar f(\uu)$ and $\pp \in \partial \psi (\K \uu)$. These are necessary and sufficient conditions for $\uu$ to be a solution of this problem~\cite[Proposition 26.11]{Bauschke2011}.}
	Let $\wt^* = \bm{\Gamma} \bm{\omega}^*$. Then, the following inequality holds:
	\begin{equation} \label{eq:lemma}
	\| \wt^i - \wt^{*} \|^2_{\G} - \| \wt^{i+1} - \wt^{*} \|^2_{\G} \geq \| \wt^{i+1} - \wt^i \|^2_{\G}.
	\end{equation}
\end{lemma} 

\begin{proof}
	
	We will give a proof for the case in which only one block-Gauss-Seidel pass is used, per iteration, to compute $\x$ and $\z$ in the partial ADMM. The proof is easily extendable to more than one pass per iteration.
	
	The necessary and sufficient conditions for the existence of $\uu^*, \vs^*$ and $\dv^*$ are the following: primal feasibility,
	\begin{equation} \label{eq:primal}
	\vs^* = \K \uu^* \Leftrightarrow \vs_x^* = \D \x^*, \vs_z^* = \z^*,
	\end{equation}
	and dual feasibility,
	\begin{equation} \label{eq:dual_x}
	- \mu \D^{T} \dv_x^* = \nabla_{\x} f(\x^*, \z^*),
	\end{equation}
	\begin{equation} \label{eq:dual_z}
	- \mu \dv_z^* = \nabla_\z f(\x^*, \z^*),
	\end{equation}
	\begin{align}
	&\mu \dv_x^* \in \partial_{\vs_x} \psi(\vs^*) \nonumber \\
	\Leftrightarrow \quad &\mu \dv_x^* \in \partial_{\vs_x} \phi(\vs_x^*)  \label{eq:dual_v_x}
	\end{align}
	\begin{equation} \label{eq:dual_v_z}
	\dv_z^* = \mathbf{0}.
	\end{equation}
	Eqs.~\eqref{eq:dual_v_x} and ~(\ref{eq:dual_v_z}) use the fact that $\psi(\vs)$ doesn't depend on $\vs_z$, and the latter equation also uses the fact that $\mu \neq 0$.

	
	The solutions of the minimization problems of the partial ADMM obey, respectively,
	\begin{align}
	\nabla_{\x} f(\x^{i+1}, \z^i) &= \mu \D^T (\vs_x^i - \D \x^{i+1} - \dv_x^i) \nonumber \\
	&= \mu \D^T (\vs_x^i - \vs_x^{i+1}) - \mu \D^{T} \dv_x^{i+1}, \label{eq:dual_x_iter} \\
	\nabla_\z f(\x^{i+1}, \z^{i+1}) &= \mathbf 0,  \label{eq:dual_z_iter} \\
	\mathbf 0 \in \partial_{\vs_x} \phi(\vs_x^{i+1}) &+ \mu (\vs_x^{i+1} - \D \x^{i+1} - \dv_x^i) \nonumber \\
	\Leftrightarrow \quad \mu \dv_x^{i+1} \in \; &\partial_{\vs_x} \phi(\vs_x^{i+1}),  \label{eq:dual_v2_iter}
	\end{align}
	where we have used~(\ref{eq:admm_aprox_4}) in~(\ref{eq:dual_x_iter}) and~(\ref{eq:dual_v2_iter}). Since $\phi$ is convex, $\partial \phi$ is monotone, from Kachurovskii's theorem. Denote the vectors $\vs_x$ and $\dv_x$ from two different iterations of the partial ADMM by $\vs_x'$ and $\dv_x'$, and $\vs_x''$ and $\dv_x''$, respectively. Due to the monotonicity of $\partial \phi$ and to \eqref{eq:dual_v2_iter}, they will satisfy the condition
	\begin{equation} \label{eq:monotone1}
	(\vs_x' - \vs_x'')^T (\dv_x' - \dv_x'') \geq 0.
	\end{equation}
	Also due to that monotonicity, now together with \eqref{eq:dual_v2_iter} and \eqref{eq:dual_v_x}, we have
	\begin{equation} \label{eq:monotone2}
	(\vs_x' - \vs_x^*)^T (\dv_x' - \dv_x^*) \geq 0.
	\end{equation}
	
	Computing the gradients in Eqs.~\eqref{eq:dual_x} and \eqref{eq:dual_x_iter}, we obtain, respectively,
	\begin{equation} \label{eq:dual_x_resolved}
	\A \x^* + \B \z^* + \e + \mu \D^T \dv_x^* = 0,
	\end{equation}
	\begin{equation} \label{eq:dual_x_iter_resolved}
	\A \x^{i+1} + \B \z^i + \e + \mu \D^T \dv_x^{i+1} - \mu \D^T (\vs_x^i - \vs_x^{i+1}) = 0,
	\end{equation}
	and for equations \eqref{eq:dual_z} and \eqref{eq:dual_z_iter} we obtain, respectively,
	\begin{equation} \label{eq:dual_z_resolved}
	\B^T \x^* + \C \z^* + \f = 0 \Leftrightarrow \z^* = - \C^{-1} (\B^T \x^* + \f),
	\end{equation}	
	\begin{equation} \label{eq:dual_z_iter_resolved}
	\z^{i+1} = - \C^{-1} (\B^T \x^{i+1} + \f);
	\end{equation}
	in \eqref{eq:dual_z_resolved} we have used \eqref{eq:dual_v_z}.
	
	Subtracting~(\ref{eq:dual_x_resolved}) from~(\ref{eq:dual_x_iter_resolved}) yields
	\begin{equation} \label{eq:x_subtract}
	\begin{aligned}
	\A &(\x^{i+1} - \x^*) + \B (\z^i - \z^* ) \\
	& - \mu \D^T (\vs_x^i - \vs_x^{i+1}) + \mu \D^T (\dv_x^{i+1} - \dv_x^*)= 0.
	\end{aligned}
	\end{equation}
	
	Using the equality $\B (\z^i - \z^*) = \B (\z^i - \z^{i+1}) + \B (\z^{i+1} - \z^*)$ and Eqs.~\eqref{eq:dual_z_resolved} and \eqref{eq:dual_z_iter_resolved}, we can rewrite~\eqref{eq:x_subtract} as 
	\begin{equation} \label{eq:x_subtract2}
	\begin{aligned}
	( \A - &\B \C^{-1} \B^T ) (\x^{i+1} - \x^*) + \B (\z^i - \z^{i+1}) \\
	& - \mu \D^T (\vs_x^i - \vs_x^{i+1}) + \mu \D^T (\dv_x^{i+1} - \dv_x^*)= 0.
	\end{aligned}
	\end{equation} 
	
	The matrix $( \A - \B \C^{-1} \B^T )$ is PSD, as noted in Section~\ref{sec:partialadmm}. Due to this fact, multiplying both sides of \eqref{eq:x_subtract2}, on the left, by $(\x^{i+1} - \x^*)^T$, leads to
	\begin{equation} \label{eq:x_ineq}
	\begin{aligned}
	0 \leq &-(\x^{i+1} - \x^*)^T \B (\z^i - \z^{i+1}) \\
	& + \mu (\x^{i+1} - \x^*)^T \D^T (\vs_x^i - \vs_x^{i+1})\\
	& - \mu (\x^{i+1} - \x^*)^T \D^T (\dv_x^{i+1} - \dv_x^*).
	\end{aligned}
	\end{equation}
	Using Eqs.~\eqref{eq:admm_aprox_4} and~\eqref{eq:primal}, we can transform the last two members of the right-hand side of \eqref{eq:x_ineq} as follows:
	\begin{equation} \label{eq:lastmembers}
	\begin{aligned}
	&\mu (\x^{i+1} - \x^*)^T \D^T (\vs_x^i - \vs_x^{i+1}) \\
	& \qquad - \mu (\x^{i+1} - \x^*)^T \D^T (\dv_x^{i+1} - \dv_x^*) \\
	& \quad = \mu (\vs_x^i - \vs_x^{i+1})^T (\D\x^{i+1} - \D\x^*)\\
	& \qquad - \mu (\dv_x^{i+1} - \dv_x^*)^T (\D\x^{i+1} - \D\x^*) \\
	& \quad = \mu (\vs_x^i - \vs_x^{i+1})^T (\vs_x^{i+1} - \vs_x^*) \\ 
	& \qquad - \mu (\vs_x^i - \vs_x^{i+1})^T (\dv_x^i - \dv_x^{i+1}) \\
	& \qquad- \mu (\dv_x^{i+1} - \dv_x^*)^T (\vs_x^{i+1} - \vs_x^*) \\
	& \qquad + \mu (\dv_x^{i+1} - \dv_x^*)^T (\dv_x^i - \dv_x^{i+1}).
	\end{aligned}
	\end{equation}
	Subtracting~\eqref{eq:dual_z_resolved} from \eqref{eq:dual_z_iter_resolved} and multiplying both sides of the result, on the right, by $\C (\z^i - \z^{i+1})$, we obtain
	\begin{equation} \label{eq:zdiff}
	(\z^{i+1} - \z^*)^T \C (\z^i - \z^{i+1}) = -(\x^{i+1} - \x^*)^T \B (\z^i - \z^{i+1}).
	\end{equation}
	Equations \eqref{eq:lastmembers}, \eqref{eq:zdiff}, (\ref{eq:monotone1}) and (\ref{eq:monotone2}), applied to~(\ref{eq:x_ineq}), lead to
	\begin{equation} \label{eq:x_ineq2}
	\begin{aligned}
	0 & \leq (\z^{i+1} - \z^*)^T \C (\z^i - \z^{i+1}) \\		
	& \quad + \mu (\vs_x^{i+1} - \vs_x^*)^T (\vs_x^i - \vs_x^{i+1})\\
	& \quad + \mu (\dv_x^{i+1} - \dv_x^*)^T (\dv_x^i - \dv_x^{i+1})\\
	& = (\wt^{i+1} - \wt^*)^T \G (\wt^i - \wt^{i+1}). \\
	\end{aligned}
	\end{equation}
	$\G$ is PD, since it is a block-diagonal matrix composed of PD blocks. We have
	\begin{equation}
	\begin{aligned}
	\|\wt^{i+1} - \wt^*\|^2_{\G} &= \|(\wt^i - \wt^*) - (\wt^i - \wt^{i+1})\|^2_{\G}\\
	& = \|\wt^i - \wt^*\|^2_{\G} + \|\wt^i - \wt^{i+1}\|^2_{\G} \\
	& \quad - 2 (\wt^i - \wt^*)^T \G (\wt^i - \wt^{i+1}).
	\end{aligned}
	\end{equation}
	Reordering these terms, we obtain
	\begin{equation} \label{eq:reordering}
	\begin{aligned} 
	\|\wt^i - \wt^*&\|^2_{\G} - \|\wt^{i+1}- \wt^*\|^2_{\G} +  \|\wt^i - \wt^{i+1}\|^2_{\G} \\
	& = 2 (\wt^i - \wt^*)^T \G (\wt^i - \wt^{i+1}).
	\end{aligned}
	\end{equation}
	From~(\ref{eq:x_ineq2}),
	\begin{align} 
	&0 \leq (\wt^{i+1} - \wt^i + \wt^i - \wt^*)^T \G (\wt^i - \wt^{i+1}) \nonumber \\
	\Leftrightarrow \quad &0 \leq - \|\wt^{i+1} - \wt^i\|^2_{\G} \nonumber \\
	&\qquad + (\wt^i - \wt^*)^T \G (\wt^i - \wt^{i+1}) \nonumber \\
	\Leftrightarrow \quad &(\wt^i - \wt^*)^T \G (\wt^i - \wt^{i+1}) \geq  \|\wt^i - \wt^{i+1}\|^2_{\G}. \label{eq:ineq3}
	\end{align}
	Finally,~(\ref{eq:reordering}) and~(\ref{eq:ineq3}) lead to
	\begin{equation*}
	\begin{aligned}
	\|\wt^i - \wt^*\|^2_{\G} - \|\wt^{i+1} - \wt^*\|^2_{\G} \geq  \|\wt^i - \wt^{i+1}\|^2_{\G}.
	\qedhere
	\end{aligned}
	\end{equation*}
\end{proof}

\begin{lemma}  \label{lemma3}
	Under the assumptions of Lemma~\ref{lemma1},
	\begin{enumerate}
		\item The sequence $\{\| \wt^i - \wt^* \|^2_{\G}\}$ is convergent.	
		
		\item The sequence $\{\| \wt^i - \wt^{i+1} \|^2_{\G}\}$ is non-increasing and converges to $0$.
		
		\item The sequence $\{\wt^i\}$ has a convergent subsequence.
		
		\item Let $\bm{\omega}^i = \bm{\Pi} \wt^i$; the sequence $\{\bm{\omega}^i\}$ has a convergent subsequence.
	\end{enumerate}
\end{lemma}

\begin{proof}
	From inequality~(\ref{eq:lemma}), we can conclude that $\| \wt^{i+1} - \wt^{*} \|^2_{\G} \leq \| \wt^i - \wt^{*} \|^2_{\G}$. This means that the sequence $\{\| \wt^i - \wt^{*} \|^2_{\G}\}$ is non-increasing. Since its elements are non-negative, the sequence is bounded.	Since the sequence is non-increasing and bounded, it is convergent. Taking the limits of both sides of \eqref{eq:lemma}, we conclude that $\| \wt^i - \wt^{i+1} \|^2_{\G} \to 0$.

	Since the sequence $\{\| \wt^i - \wt^{*} \|^2_{\G}\}$ is bounded, $\{\wt^i\}$ is bounded as well. Therefore, it has a convergent subsequence. The same applies to $\{\bm{\omega}^i\}$.
\end{proof}

\begin{lemma} \label{lemma4}
	Under the assumptions of Lemma~\ref{lemma1}, the limit of any convergent subsequence of $\{\bm{\omega}^i\}$ is a KKT pair of problem~\eqref{eq:problem_xz}.
\end{lemma}
	
\begin{proof}
	From Lemma~\ref{lemma3}, $\|\wt^{i+1} - \wt^i \|^2_{\G} \to 0$. This implies that $\| \z^i - \z^{i+1} \|^2_\C \to 0$, $\| \vs_x^i - \vs_x^{i+1} \|^2 \to 0$, and $\| \dv_x^i - \dv_x^{i+1} \|^2 \to 0$. From the latter and from (\ref{eq:admm_aprox_4}), we have
	\begin{equation} \label{eq:lemma4}
	\vs_x^{i+1} - \D \x^{i+1} \to 0.
	\end{equation}
	From Lemma~\ref{lemma3}, $\{\wt^i\}$ has a convergent subsequence. Consider one such subsequence, $\{\wt^{i_j}\}$, whose limit we designate by $\wt^{\infty}$. Denote by $\z^{\infty}$, $\vs_x^{\infty}$, and $\dv_x^{\infty}$, respectively, the limits of the corresponding subsequences $\{\z^{i_j}\}$, $\{\vs_x^{i_j}\}$, and $\{\dv_x^{i_j}\}$. Additionally, designate by $\bm{\omega}^{\infty}$ the limit of the subsequence $\{\bm{\omega}^{i_j}\} = \{\bm{\Pi} \wt^{i_j}\}$. Note that $\z^{i_j+1} \to \z^\infty$ because, according to Lemma 2, $\|\z^i - \z^{i+1}\|_\C \to \mathbf 0$; for similar reasons, $\vs_x^{i_j+1} \to \vs_x^\infty$ and $\dv_x^{i_j+1} \to \dv_x^\infty$. Since $\{\vs_x^{i_j}\}$ converges and $\D$ is full column rank, \eqref{eq:lemma4} shows that $\{\x^{i_j}\}$ also converges. Denote its limit by $\x^\infty$. Taking the limits of both sides of \eqref{eq:lemma4} over the subsequence of indexes $\{i_j\}$, we obtain
	\begin{equation} \label{eq:primal_infty}
	\vs_x^{\infty} = \D \x^{\infty}.
	\end{equation}
	
	Taking the limits on both sides of Eqs. \eqref{eq:dual_x_iter}--\eqref{eq:dual_v2_iter}, again over the subsequence of indexes $\{i_j\}$, and applying~\cite[Theorem 24.4]{Rockafellar1970}, we have
	\begin{equation}
	- \mu \D^{T} \dv_x^{\infty} = \nabla_{\x} f(\x^{\infty}, \z^{\infty}),
	\end{equation}
	\begin{equation}
	\mathbf 0 = \nabla_\z f(\x^{\infty}, \z^{\infty}),
	\end{equation}
	\begin{equation} \label{eq:dx_inf}
	\mu \dv_x^{\infty} \in \partial_{\vs_x} \psi(\vs_x^{\infty}).
	\end{equation}
	
	Variables $\vs_z^i$ and $\dv_z^i$ are not used in the partial ADMM, and therefore we can give them any values that are convenient. By making $\bm{\omega}^i = \bm{\Pi} \wt^i$, we are setting, for all $i$, $\vs_z^i = \z^i$ and $\dv_z^i = \mathbf 0$. This implies that $\vs_z^{\infty} = \z^{\infty}$ and $\dv_z^{\infty} = \mathbf 0$. These two equalities, together with~\eqref{eq:primal_infty}--\eqref{eq:dx_inf}, correspond to the optimality conditions of problem~(\ref{eq:problem_xz}), i.e., $\bm{\omega}^{\infty}$ is a KKT pair of this problem.
\end{proof}

\begin{lemma} \label{lemma5}
	Under the assumptions of Lemma \ref{lemma1}, the sequence $\{ \bm{\omega}^i \}$ is convergent.
\end{lemma}

\begin{proof}
	We follow the argument made in, e.g.,~\cite{Rockafellar1976, Chen1994}. Consider the limits of two convergent subsequences of $\{\bm{\omega}^i\}$, $\bm{\omega}^{\infty}_1$ and $\bm{\omega}^{\infty}_2$. According to Lemma~\ref{lemma4}, $\bm{\omega}^{\infty}_1$ and $\bm{\omega}^{\infty}_2$ are both KKT pairs of problem~(\ref{eq:problem_u}). Therefore, according to Lemma \ref{lemma3}, the following two limits exist:
	\begin{equation}
	\lim_{i \to \infty} \| \bm{\omega}^i - \bm{\omega}^{\infty}_p \|_{\G} = a_p, \quad p = 1,2.
	\end{equation}
	We have
	\begin{align*}
	\| \bm{\omega}^i& - \bm{\omega}^{\infty}_1 \|^2_{\G} - \| \bm{\omega}^i - \bm{\omega}^{\infty}_2 \|^2_{\G} \\
	& = -2 (\bm{\omega}^i)^T(\bm{\omega}^{\infty}_1 - \bm{\omega}^{\infty}_2) + \| \bm{\omega}^{\infty}_1 \|^2_{\G}  - \| \bm{\omega}^{\infty}_2 \|^2_{\G}.
	\end{align*}
	Taking the limits of both sides over a subsequence of $\{\bm{\omega}^i\}$ that converges to $\bm{\omega}^{\infty}_1$, we have 
	\begin{align*}
	a_1^2 - a_2^2 & = -2 (\bm{\omega}^{\infty}_1)^T(\bm{\omega}^{\infty}_1 - \bm{\omega}^{\infty}_2) + \| \bm{\omega}^{\infty}_1 \|^2_{\G}  - \| \bm{\omega}^{\infty}_2 \|^2_{\G} \\
	& = - \| \bm{\omega}^{\infty}_1 - \bm{\omega}^{\infty}_2 \|^2_{\G},
	\end{align*}
	and taking the limits over a subsequence that converges to $\bm{\omega}^{\infty}_2$, we have
	\begin{align*}
	a_1^2 - a_2^2 & = -2 (\bm{\omega}^{\infty}_2)^T(\bm{\omega}^{\infty}_1 - \bm{\omega}^{\infty}_2) + \| \bm{\omega}^{\infty}_1 \|^2_{\G}  - \| \bm{\omega}^{\infty}_2 \|^2_{\G} \\
	& = \| \bm{\omega}^{\infty}_1 - \bm{\omega}^{\infty}_2 \|^2_{\G}.
	\end{align*}
	These two equalities imply that $\| \bm{\omega}^{\infty}_1 - \bm{\omega}^{\infty}_2 \|^2_{\G} = 0$. This means that $\bm{\omega}^{\infty}_1=\bm{\omega}^{\infty}_2$. Therefore, all convergent subsequences of $\{\bm{\omega}^i\}$ converge to the same limit, and as a consequence $\{\bm{\omega}^i\}$ is convergent.
\end{proof}

After these preliminary results, we prove Theorem \ref{th:main}.

%

\begin{proof}[Proof of Theorem 1]
	With problem~\eqref{eq:problem_xz} reformulated as~\eqref{eq:problem_u}, we are under the conditions of Lemmas~\ref{lemma1}--\ref{lemma5}. Lemma~\ref{lemma1} shows that the set of solutions of the problem is non-empty. Lemma~\ref{lemma5} allows us to conclude that the sequence $\{\bm{\omega}^i\}$ converges. Lemma~\ref{lemma4} shows that its limit corresponds to a KKT pair. This implies that $\big\{\big[\begin{smallmatrix} \x^i \\ \z^i \end{smallmatrix}\big]\big\}$ converges to a solution of problem~\eqref{eq:problem_xz}, $\big[\begin{smallmatrix} \x^* \\ \z^* \end{smallmatrix}\big]$, and that $\{\vs^i\}$ converges to $\K \big[\begin{smallmatrix} \x^* \\ \z^* \end{smallmatrix}\big]$.
\end{proof}

\section{Additional experimental results}
\label{sec:app_exp}

This Appendix presents the computing times for some of the experiments described in \mbox{Section~\ref{sec:deblurring}}.

\begin{table}[h]
	\renewcommand{\arraystretch}{1}
	\caption{Results for \textit{cameraman} with boxcar blurs of various sizes.}
	\centering
\begin{tabular}{l|c|c|c|c}
Method &Blur size &$\kappa \times 10^3$ &Iterations &Time (s) \\
\hline
Proposed-1 & \multirow{5}{*}{5} & \multirow{5}{*}{57.0} &   74 & 0.796  \\ 
Proposed-AD & & &   41 & 0.600 \\ 
AM & & &  179 & 1.977 \\ 
ADMM-CG & & &   25 & 20.416 \\ 
CM & & & 12028 & 52.405 \\ 
\hline 
Proposed-1 & \multirow{5}{*}{13} & \multirow{5}{*}{492.4} &  457 & 4.030  \\ 
Proposed-AD & & &  133 & 3.072 \\ 
AM & & &  401 & 4.376 \\ 
ADMM-CG & & &   49 & 92.345 \\ 
CM & & & 58623 & 259.226 \\ 
\hline 
Proposed-1 & \multirow{5}{*}{21} & \multirow{5}{*}{1356.5} &  410 & 3.628  \\ 
Proposed-AD & & &   95 & 3.244 \\ 
AM & & &  492 & 5.336 \\ 
ADMM-CG & & &   65 & 142.542 \\ 
CM & & & 122902 & 535.551 \\ 
\hline 
\end{tabular}
\end{table}

\begin{table}[h]
	\renewcommand{\arraystretch}{1}
	\caption{Results for \textit{Lena} with boxcar blurs of various sizes.}
	\centering
\begin{tabular}{l|c|c|c|c}
Method &Blur size &$\kappa \times 10^3$ &Iterations &Time (s) \\
\hline
Proposed-1 & \multirow{5}{*}{5} & \multirow{5}{*}{57.0} &  131 & 1.256  \\ 
Proposed-AD & & &   52 & 0.721 \\ 
AM & & &  264 & 3.032 \\ 
ADMM-CG & & &   25 & 21.188 \\ 
CM & & & 11353 & 49.062 \\ 
\hline 
Proposed-1 & \multirow{5}{*}{13} & \multirow{5}{*}{492.4} & 1038 & 8.908  \\ 
Proposed-AD & & &  140 & 3.192 \\ 
AM & & &  399 & 4.390 \\ 
ADMM-CG & & &   49 & 86.374 \\ 
CM & & & 51328 & 223.575 \\ 
\hline 
Proposed-1 & \multirow{5}{*}{21} & \multirow{5}{*}{1356.5} & 2202 & 18.713  \\ 
Proposed-AD & & &  176 & 5.761 \\ 
AM & & &  552 & 6.038 \\ 
ADMM-CG & & &   65 & 138.104 \\ 
CM & & & 118684 & 517.376 \\ 
\hline 
\end{tabular}
\end{table}

\begin{table}[h]
	\renewcommand{\arraystretch}{1}
	\caption{Results for \textit{Lena} with Gaussian blurs of various sizes.}
	\centering
\begin{tabular}{l|c|c|c|c}
Method &Blur size &$\kappa \times 10^3$ &Iterations &Time (s) \\
\hline
Proposed-1 & \multirow{5}{*}{5} & \multirow{5}{*}{247.0} &  196 & 1.828  \\ 
Proposed-AD & & &   73 & 0.958 \\ 
AM & & &  331 & 3.737 \\ 
ADMM-CG & & &   35 & 26.204 \\ 
CM & & & 13946 & 61.351 \\ 
\hline 
Proposed-1 & \multirow{5}{*}{13} & \multirow{5}{*}{21590.6} & 6093 & 52.058  \\ 
Proposed-AD & & &  776 & 16.905 \\ 
AM & & & 1032 & 11.344 \\ 
ADMM-CG & & &   81 & 120.265 \\ 
CM & & & 196053 & 867.026 \\ 
\hline 
Proposed-1 & \multirow{5}{*}{21} & \multirow{5}{*}{686384.8} & 68776 & 583.065  \\ 
Proposed-AD & & & 7047 & 233.340 \\ 
AM & & & 1454 & 15.863 \\ 
ADMM-CG & & &  205 & 345.285 \\ 
CM & & & 412249 & 1833.467 \\ 
\hline 
\end{tabular}
\end{table}



\ifCLASSOPTIONcaptionsoff
  \newpage
\fi



\bibliographystyle{IEEEtran}
\bibliography{IEEEabrv,refs}
%

%


\begin{IEEEbiography}[{\includegraphics[width=1in,height=1.25in,clip,keepaspectratio]{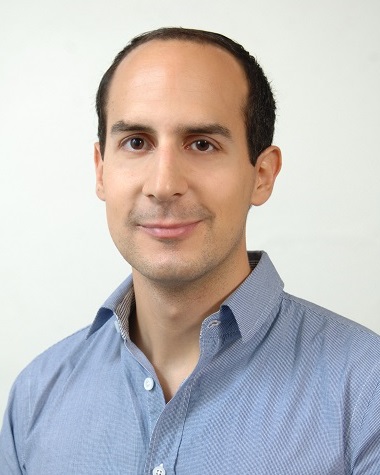}}]{Miguel Sim\~{o}es}
	received the M.Sc. degree in electrical and computer engineering from the Instituto Superior T\'{e}cnico, University of Lisbon, Lisbon, Portugal, in 2010. He is currently working toward the joint Ph.D. degree in electrical and computer engineering, and signal and image processing at the Instituto de Telecomunica\c{c}\~{o}es, Instituto Superior T\'{e}cnico, University of Lisbon, Lisbon, and at the Grenoble Images Parole Signal Automatique (GIPSA-lab), University of Grenoble, Grenoble, France, respectively.
	
	He has previousl yworked as an Information Technology Consultant in the field of telecommunications. His main areas of research interest are image processing, optimization, and remote sensing.
\end{IEEEbiography}

\begin{IEEEbiography}[{\includegraphics[width=1in,height=1.25in,clip,keepaspectratio]{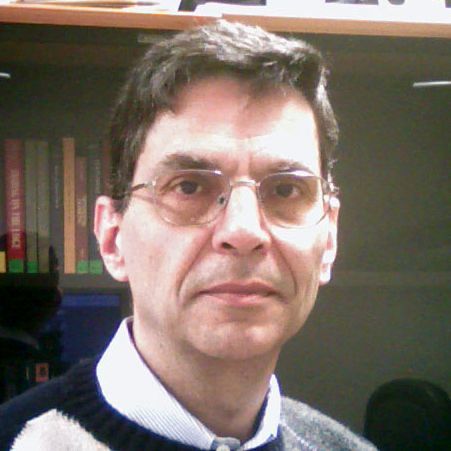}}]{Lu\'is B. Almeida} 
	received the “Doutor” degree from the Techincal University of Lisbon, Portugal, in 1983. Since 1972, he has been a faculty member of Instituto Superior T\'{e}cnico (IST), Lisbon University, where he has been a full professor since 1995, lecturing in the areas of signal processing and machine learning. From 1984 to 2004, he was head of the Neural Networks and Signal Processing Group of INESC-ID. From 2000 to 2003, he was chair of INESC-ID. In 2005, he joined the Instituto de Telecomunica\c{c}\~{o}es (Telecommunications Institute). From 2008 to 2010, he was chair of the Electrical and Computer Engineering Department of IST. Over the years, he has worked on speech modelling and coding, time-frequency representations of signals and the fractional Fourier transform, learning algorithms for neural networks, blind source separation, and, currently, image processing. He was the recipient of an IEEE Signal Processing Area ASSP Senior Award and of several national awards.
\end{IEEEbiography}

\begin{IEEEbiography}[{\includegraphics[width=1in,height=1.25in,clip,keepaspectratio]{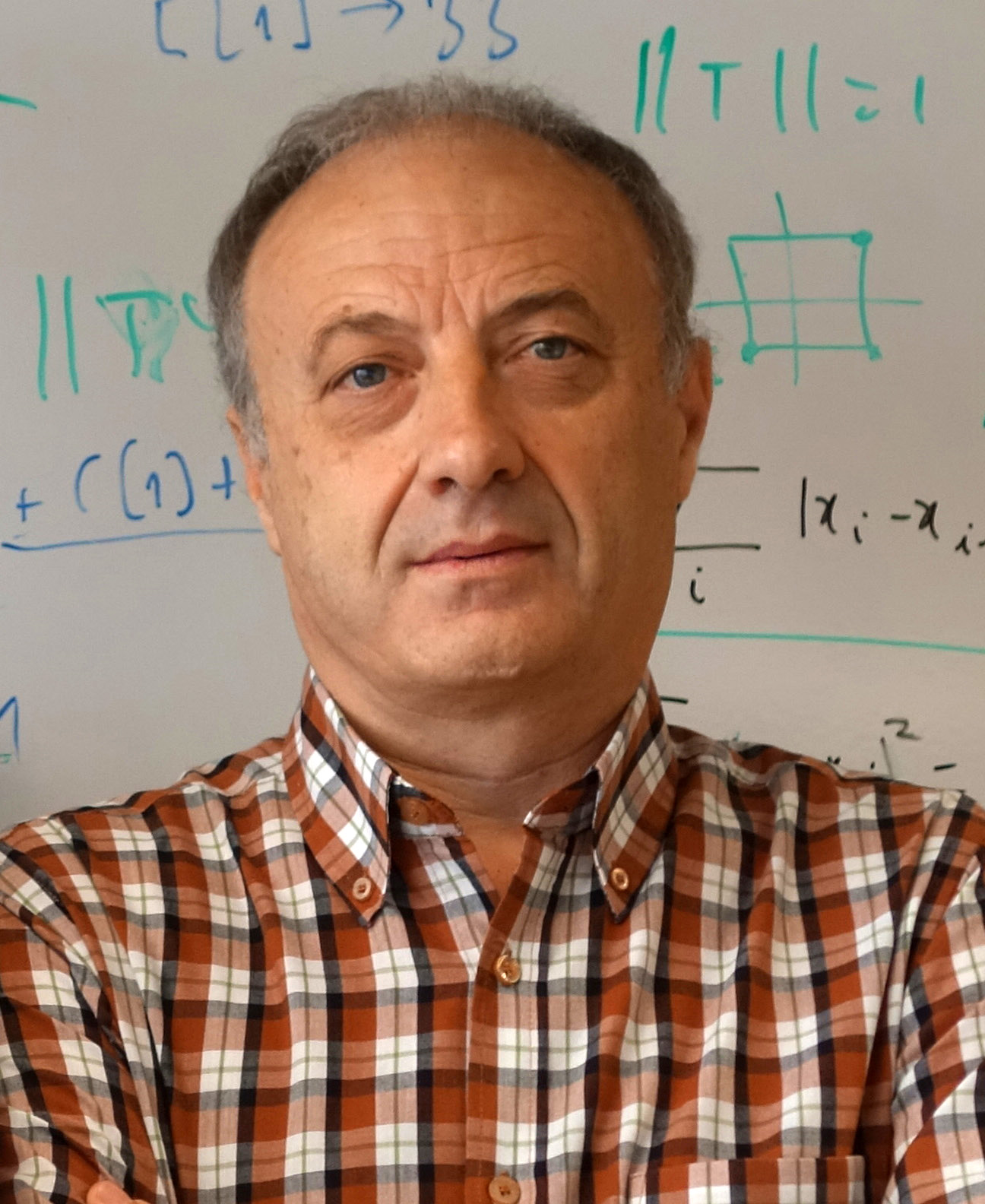}}]{Jos\'e Bioucas-Dias}
	(SM'15) received the EE, MSc, PhD, and ``Agregado" degrees from Instituto Superior T\'ecnico (IST), Technical University of Lisbon (TULisbon, now University of Lisbon), Portugal, in 1985, 1991, 1995, and 2007, respectively, all in electrical and computer engineering.
	
	Since 1995, Jos\'e Bioucas-Dias has been with the Department of Electrical and Computer Engineering, IST, where he was an Assistant Professor from 1995 to 2007 and an Associate Professor since 2007. Since 1993, he is also a Senior Researcher with the Pattern and Image Analysis group of the Instituto de Telecomunica\c{c}\~oes, which is a private  non-profit research institution. His  research interests include inverse problems, signal and image processing, pattern recognition, optimization, and remote sensing. His research work has been highly cited and he is included in Thomson Reuters' Highly Cited Researchers 2015 list. 
	
	Jos\'e Bioucas-Dias was an Associate Editor for the {\sc IEEE Transactions on Circuits and Systems} (1997-2000) and  {\sc IEEE Transactions on Image Processing} (2010-2014) and he is  a Senior Area  Editor for the {\sc IEEE Transactions on Image Processing} and an Associate Editor for the {\sc IEEE Transactions on Geoscience and Remote Sensing}.  He was the General Co-Chair of the 3rd IEEE GRSS Workshop on Hyperspectral Image and Signal Processing, Evolution in Remote sensing (WHISPERS'2011) and has been a member of program/technical committees of several international conferences. 
\end{IEEEbiography}

\begin{IEEEbiography}[{\includegraphics[width=1in,height=1.25in,clip,keepaspectratio]{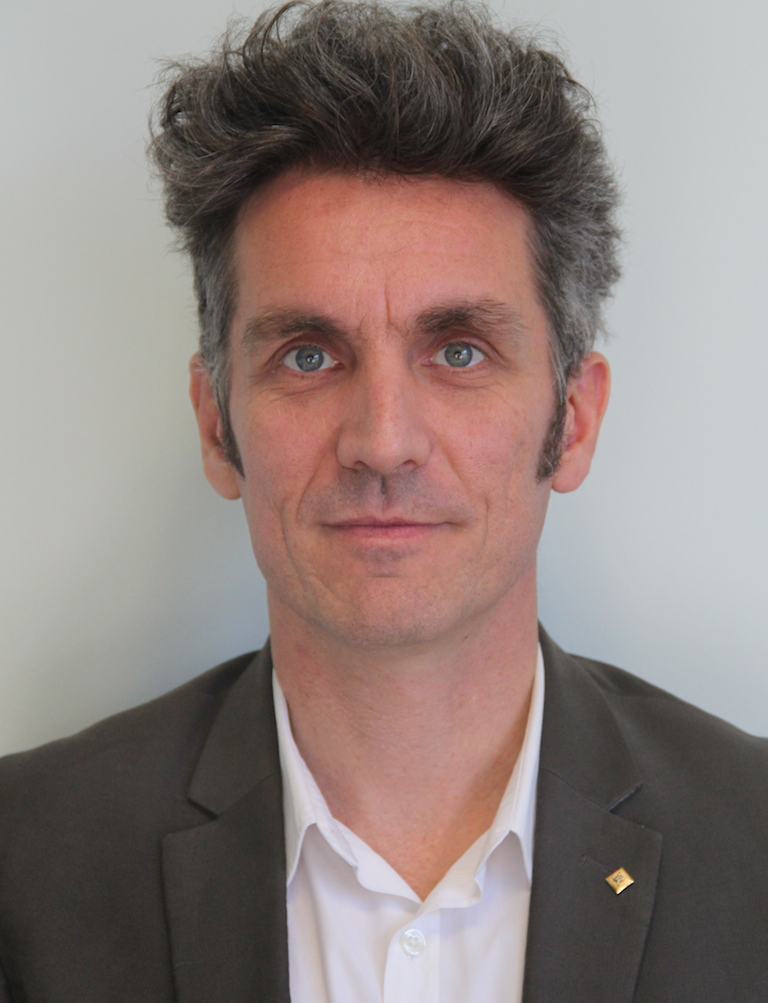}}]{Jocelyn Chanussot} 
	(M’04–SM’04–F’12) received the M.Sc. degree in electrical engineering from the Grenoble Institute of Technology (Grenoble INP), Grenoble, France, in 1995, and the Ph.D. degree from the Université de Savoie, Annecy, France, in 1998. In 1999, he was with the Geography Imagery Perception Laboratory for the Delegation Generale de l'Armement (DGA - French National Defense Department). Since 1999, he has been with Grenoble INP, where he was an Assistant Professor from 1999 to 2005, an Associate Professor from 2005 to 2007, and is currently a Professor of signal and image processing. He is conducting his research at the Grenoble Images Speech Signals and Automatics Laboratory (GIPSA-Lab). His research interests include image analysis, multicomponent image processing, nonlinear filtering, and data fusion in remote sensing. He has been a visiting scholar at Stanford University (USA), KTH (Sweden) and NUS (Singapore). Since 2013, he is an Adjunct Professor of the University of Iceland. In 2015-2017, he is a visiting professor at the University of California, Los Angeles (UCLA).  
	
	Dr. Chanussot is the founding President of IEEE Geoscience and Remote Sensing French chapter (2007-2010) which received the 2010 IEEE GRS-S Chapter Excellence Award. He was the co-recipient of the NORSIG 2006 Best Student Paper Award, the IEEE GRSS 2011 and 2015 Symposium Best Paper Award, the IEEE GRSS 2012 Transactions Prize Paper Award and the IEEE GRSS 2013 Highest Impact Paper Award. He was a member of the IEEE Geoscience and Remote Sensing Society AdCom (2009-2010), in charge of membership development. He was the General Chair of the first IEEE GRSS Workshop on Hyperspectral Image and Signal Processing, Evolution in Remote sensing (WHISPERS). He was the Chair (2009-2011) and  Cochair of the GRS Data Fusion Technical Committee (2005-2008). He was a member of the Machine Learning for Signal Processing Technical Committee of the IEEE Signal Processing Society (2006-2008) and the Program Chair of the IEEE International Workshop on Machine Learning for Signal Processing, (2009). He was an Associate Editor for the IEEE Geoscience and Remote Sensing Letters (2005-2007) and for Pattern Recognition (2006-2008). Since 2007, he is an Associate Editor for the IEEE Transactions on Geoscience and Remote Sensing. He was the Editor-in-Chief of the IEEE Journal of Selected Topics in Applied Earth Observations and Remote Sensing (2011-2015). In 2013, he was a Guest Editor for the Proceedings of the IEEE and in 2014 a Guest Editor for the IEEE Signal Processing Magazine. He is a Fellow of the IEEE and a member of the Institut Universitaire de France (2012-2017).	
\end{IEEEbiography}

%
%
%





\end{document}